\documentclass[preprint,3p,authoryear,nopreprintline]{elsarticle}
\usepackage[shortlabels]{enumitem}
\usepackage{graphicx}
\usepackage{listings}
\usepackage{booktabs}
\usepackage{algorithm}
\usepackage{algpseudocode}
\usepackage{amsmath,amsthm,amssymb}
\usepackage{mathtools}
\usepackage{url}
\usepackage{subcaption}
\usepackage{multirow}
\usepackage[dvipsnames]{xcolor}
\usepackage{tikz}
\usetikzlibrary{shapes, arrows.meta, positioning, calc, backgrounds, fit}

\newtheorem{theorem}{Theorem}

\newtheorem{proposition}{Proposition}
\newtheorem{example}{Example}

\newcommand{\OMIT}[1]{}

\newcommand{\E}{\mathbb{E}}
\newcommand{\Perror}{\mathbb{P}_{\mathrm{error}}}

\renewcommand{\cite}{\citep}

\usepackage{lineno}

\usepackage{marginnote} 
\usepackage{geometry}   
\usepackage[normalem]{ulem} 

\newif\ifshowchanges
\showchangestrue 

\newcommand{\new}[2][]{#2}

\newcommand{\neww}[2][]{#2}

\begin{document}

\begin{frontmatter}



\author{Chih-Duo Hong}

\author{Yen-Pang Chen}

\author{Fang Yu}

\affiliation{organization={National Chengchi University},city={Taipei},country={Taiwan}}
\cortext[cor1]{Chih-Duo Hong is the corresponding author. This work was partially supported by the National Science and Technology Council (NSTC), Taiwan, under grant no.~112-2222-E-004-001-MY3, 113-2221-E-004-010-MY2, and 114-2634-F-004-002-MBK.}

\title{Signature filtering: A lightweight enhancement for statistical watermark detection in large language models}





\begin{abstract}
Statistical watermarks help organizations attribute large language model (LLM) outputs, yet existing detectors often struggle when watermark signals are weak, texts are repetitive, or watermarks are edited. We propose \emph{signature filtering}, a detection-time module that enhances watermark detection without modifying watermark embedding and text generation.
It learns a small set of ``signature'' tokens whose presence makes watermark tests unreliable, and removes these tokens before detection. The signatures are obtained by solving a mixed‑integer linear program on a small training set, with constraints that maximize the true positive rate. We additionally derive finite‑sample and asymptotic bounds under several attacker models (color-blind, color-adaptive, and distributionally correlated). On four well-known watermark families (Kgw, Sweet, Unigram, Exp), four benchmark corpora (C4, MBPP, HumanEval, Code‑Search‑Net), and six LLMs (Opt‑1.3b, Opt-6.7b, Llama2‑13b, Llama3.1‑8b, Qwen2.5‑14b, Phi‑3‑medium‑14b), 2--3-gram signatures raise detection rates in weak-signal and low-entropy settings from 8--31\% without filtering to 78--99\% with filtering, while keeping false positives controllable and often negligible. In stress tests where we scramble sentences and perturb 25--50\% of tokens by dilution, deletions, and substitutions, 2-gram filters for Kgw-style watermarks preserve most of the clean-text detection gains, often matching or outperforming the advanced WinMax watermark detector. Signature filtering thus provides a simple, scalable, and model‑agnostic add-on to strengthen watermark-based provenance checks for LLM text in information processing workflows.
\end{abstract}


\end{frontmatter}

\section{Introduction}
Large language models (LLMs) now generate vast amounts of text for both public web services and internal enterprise applications. They power summarization, assistance, and content creation across many domains~\cite{xylogiannopoulos2024chatgpt}. As these outputs are mixed with human-authored material in search indexes, content feeds, and institutional repositories, organizations must routinely decide whether a passage is likely AI-generated. These decisions affect reliability, regulatory compliance, and user trust in information services~\cite{xiong2025delphiagent,wang2025benfordipm}. Text watermarking addresses this need by embedding imperceptible statistical signals during generation that can later be verified, providing a practical mechanism for attribution, auditing, and governance in information-processing pipelines~\cite{liu2024survey,wu2025survey}. However, today's detectors often struggle when (i) watermark signals must remain weak to preserve text quality, (ii) the text itself is highly repetitive or formulaic, or (iii) watermarked passages are heavily mixed or edited~\cite{wu2025survey}.

We propose \emph{signature filtering} to address these challenges. The idea is simple: before running the baseline hypothesis test, we remove a pre-computed set of ``statistically disruptive'' tokens from the text. These tokens form a signature that can be learned offline from historical model outputs, online during text generation, or incrementally on a streaming corpus. Removing them increases the separation between watermarked and natural text in borderline cases and recovers many true positives in weak-signal and low-entropy regimes. Figure~\ref{fig:signature-4panel} illustrates this effect under \textsc{Kgw}: signature filtering expands a borderline z-score gap into a decisive one, increasing the gap between the two passages from 1.77 to 5.01 and allowing the detector to correctly identify an otherwise undetectable watermark.

\definecolor{myred}{RGB}{204, 0, 0}
\definecolor{mygreen}{RGB}{0, 176, 80}
\begin{figure}[t]
\resizebox{\columnwidth}{!}{%
\begin{tabular}{cc}
\begin{minipage}[t]{0.48\columnwidth}
\fontsize{8}{10}\selectfont
\textcolor{myred}{People} \textcolor{mygreen}{with} \textcolor{mygreen}{marketing} \textcolor{myred}{backgrounds} \textcolor{mygreen}{are} \textcolor{myred}{hired} \textcolor{myred}{by} \textcolor{mygreen}{government} \textcolor{myred}{agencies} \textcolor{mygreen}{in} \textcolor{myred}{a} \textcolor{myred}{number} \textcolor{mygreen}{of} \textcolor{myred}{capacities.}
\textcolor{mygreen}{Government} \textcolor{myred}{agencies} \textcolor{mygreen}{at} \textcolor{myred}{the} \textcolor{myred}{local,} \textcolor{mygreen}{state} \textcolor{myred}{and} \textcolor{mygreen}{federal} \textcolor{myred}{level} \textcolor{mygreen}{all} \textcolor{myred}{employ} \textcolor{mygreen}{marketing} \textcolor{myred}{professionals} \textcolor{mygreen}{in} \textcolor{myred}{areas} \textcolor{mygreen}{including,} \textcolor{myred}{but} \textcolor{mygreen}{not} \textcolor{myred}{limited} \textcolor{mygreen}{to,} \textcolor{myred}{public} \textcolor{mygreen}{relations,} \textcolor{myred}{property} \textcolor{mygreen}{disposal,} \textcolor{myred}{bond} \textcolor{mygreen}{sales} \textcolor{mygreen}{and} \textcolor{myred}{purchasing.}

\leavevmode\\\vspace{1em}{\centering\textbf{\small (a) Unwatermarked text}\par\vspace{0.5em}}
\end{minipage}
&
\begin{minipage}[t]{0.48\columnwidth}
\fontsize{8}{10}\selectfont
\textcolor{myred}{People} \textcolor{mygreen}{with} \textcolor{mygreen}{marketing} \textcolor{myred}{backgrounds} \textcolor{mygreen}{are} \textcolor{myred}{hired} \textcolor{myred}{by} government agencies \textcolor{mygreen}{in} \textcolor{myred}{a} \textcolor{myred}{number} \textcolor{mygreen}{of} \textcolor{myred}{capacities.}
\textcolor{mygreen}{Government} agencies \textcolor{mygreen}{at} the local\textcolor{myred}{,} \textcolor{mygreen}{state} \textcolor{myred}{and} federal \textcolor{myred}{level} \textcolor{mygreen}{all} \textcolor{myred}{employ} \textcolor{mygreen}{marketing} \textcolor{myred}{professionals} \textcolor{mygreen}{in} areas \textcolor{mygreen}{including,} \textcolor{myred}{but} \textcolor{mygreen}{not} \textcolor{myred}{limited} to\textcolor{mygreen}{,} \textcolor{myred}{public} relations\textcolor{mygreen}{,} \textcolor{myred}{property} \textcolor{mygreen}{disposal,} bond \textcolor{mygreen}{sales} \textcolor{mygreen}{and} purchasing\textcolor{myred}{.}

\leavevmode\\\vspace{.5em}{\centering\textbf{\small (b) Unwatermarked~and filtered}\par\vspace{1em}}
\end{minipage}
\\[0em] 
\begin{minipage}[t]{0.48\columnwidth}
\fontsize{8}{10}\selectfont
\textcolor{mygreen}{Public} \textcolor{myred}{sector} \textcolor{mygreen}{marketers} \textcolor{myred}{work} \textcolor{mygreen}{for} \textcolor{myred}{government} \textcolor{mygreen}{agencies,} \textcolor{mygreen}{non}\textcolor{mygreen}{-profit} \textcolor{myred}{organizations,} \textcolor{mygreen}{and} \textcolor{myred}{other} \textcolor{mygreen}{public} \textcolor{myred}{institutions.} \textcolor{mygreen}{Their} \textcolor{myred}{goal} \textcolor{mygreen}{is} \textcolor{myred}{to} \textcolor{mygreen}{promote} \textcolor{myred}{the} \textcolor{mygreen}{services,} \textcolor{mygreen}{products,} \textcolor{mygreen}{and} \textcolor{myred}{policies} \textcolor{mygreen}{of} \textcolor{myred}{their} \textcolor{mygreen}{organization} \textcolor{myred}{to} \textcolor{mygreen}{the} \textcolor{myred}{public.}

\leavevmode\\\vspace{.5em}{\centering\textbf{\small (c) Watermarked text}\par\vspace{0em}}
\end{minipage}
&
\begin{minipage}[t]{0.48\columnwidth}
\fontsize{8}{10}\selectfont
\textcolor{mygreen}{Public} \textcolor{myred}{sector} \textcolor{mygreen}{marketers} \textcolor{myred}{work} \textcolor{mygreen}{for} government \textcolor{mygreen}{agencies,} \textcolor{mygreen}{non}\textcolor{mygreen}{-profit} organizations\textcolor{myred}{,} \textcolor{mygreen}{and} \textcolor{myred}{other} \textcolor{mygreen}{public} institutions\textcolor{mygreen}{.} \textcolor{mygreen}{Their} \textcolor{myred}{goal} \textcolor{mygreen}{is} to \textcolor{mygreen}{promote} the \textcolor{mygreen}{services,} \textcolor{mygreen}{products,} \textcolor{mygreen}{and} \textcolor{myred}{policies} \textcolor{mygreen}{of} their \textcolor{mygreen}{organization} to \textcolor{mygreen}{the} \textcolor{myred}{public.}

\leavevmode\\\vspace{.5em}{\centering\textbf{\small (d) Watermarked~and filtered}\par\vspace{0em}}
\end{minipage}
\end{tabular}} 
\caption{Unwatermarked and watermarked texts with and without signature filtering. Words containing filtered tokens are rendered in black.
    (a) $z$ = $-0.14$, Red = 50\%, Green = 49\%;
    (b) $z$ = 1.66, Red = 29\%, Green = 38\%, Filtered = 33\%;
    (c) $z$ = 1.63, Red = 44\%, Green = 56\%; 
    (d) $z$ = 6.67, Red = 9\%, Green = 42\%, Filtered = 49\%. 
    \neww{Observe that filtering impacts the scores differently: it boosts the watermarked z-score from 1.63 to 6.67, but slightly inflates the unwatermarked z-score from -0.14 to 1.66. This filtering step therefore flips the detector decision \textit{only} for the watermarked~text w.r.t.~the threshold $z_0=4$.}}
\label{fig:signature-4panel}
\end{figure}

\new{From a statistical perspective, a natural question is whether deleting tokens based on the observed text invalidates the null-distribution assumptions behind the underlying watermark test. In our setting, it does not.
\neww{Under the \textsc{Kgw} coloring model, token colors in an unwatermarked text are i.i.d.~random variables conditioned on the secret hash key.} Deleting any subset of tokens that is chosen independently of those hidden colors preserves the distribution on the retained tokens, so the usual z-test remains valid.
For \textsc{Exp}-style watermarks, the detector assigns each realized token a score that forms an i.i.d.~random variable in an unwatermarked text. Since our filter is a deterministic function of the observed text, the retained scores remain i.i.d., and the null distribution is correctly calibrated after filtering.
Thus, for unwatermarked texts, running the baseline test on the filtered text is as statistically valid as running the test on the original text, \neww{provided the deletion rule relies only on the observable features} and the filtered text is sufficiently large.}

Another concern is whether any technique built on independence assumptions can maintain its performance on correlated text. Our empirical evaluation suggests that signature filtering remains effective even when token coloring is far from independent. First, on low-entropy code corpora, where limited variation breaks many detectors, signature filtering achieves near-optimal true positive rates (TPR) for \textsc{Kgw} at negligible false positive rates (FPR). Second, when we apply it to more generic watermark schemes where the \textsc{Kgw} coloring assumption no longer holds, signature filtering still consistently improves detection compared with the baseline. Third, under common sentence- and word-level watermark removal attacks, 2-gram signature filtering can match or outperform state-of-the-art attack-resistant detectors. Taken together, these results show that our method is robust in non-idealized and correlated settings.

Although a signature is learned from a training dataset, its performance generalizes reasonably well in our experiments. Signatures learned on about 1{,}000 training texts maintain strong TPR and negligible FPR when applied to 50{,}000 previously unseen texts. We further show that computing separate signatures on successive text batches and aggregating their scores preserves near-oracle TPR while keeping the combined FPR well below the union bound. Signature generation therefore incurs a one-time or amortized cost at deployment time, after which the learned filters scale gracefully to much larger or entirely new datasets.

\medskip
\noindent\textbf{Research objectives and contributions.}
Motivated by the fragility of existing LLM watermark detectors in weak-signal, low-entropy, and edited-text settings, this study pursues three objectives:


\begin{itemize}
\item \textit{Investigate detection-time filtering.} We study whether discarding a small pre-learned subset of tokens before running a standard watermark test can increase detection power. 

\item \textit{Characterize statistical validity.} We seek analytical conditions under which our enhancement method remains reliable when texts are correlated, partially edited, or adversarially crafted.

\item \textit{Assess effectiveness and scalability.} We evaluate performance across various practical settings, studying how our method can be deployed in large-scale information processing pipelines effectively.
\end{itemize}
To address these objectives, this work makes the following main contributions:
\begin{itemize}
    \item \textit{Design a detection-time enhancement for LLM watermarks.} We develop signature filtering as a plug-in module for existing statistical watermarking schemes. It is designed to improve watermark detection in weak-signal, low-variation, and edited-text scenarios where existing techniques often struggle.

    \item \textit{Provide formal statistical guarantees under threat models.} We analyze when signature filtering preserves the nominal Type-I error of standard z-tests and derive finite-sample and asymptotic worst-case bounds on false positives under several adversarial and dependency models.

    \item \textit{Evaluate detection capability on realistic attacks and datasets.} We assess signature filtering across watermark families, language models, and corpora benchmarks, \neww{as well as quantifying the sensitivity of filtering efficacy to signature training sizes.} The results show that signatures can enhance detection, resist a range of text edits, generalize via predictive reuse, and scale effectively to streaming texts.
\end{itemize}

Our approach differs from existing watermark-enhancing techniques in several fundamental ways. Prior enhancements typically modify the embedding procedure, redesign the test statistic (e.g., by maximizing over sliding windows), or introduce new entropy- or semantics-based scores that must be calibrated alongside watermark tests. Our method keeps the underlying watermark family and z-test unchanged and instead learns a compact filter via optimization on representative data. This design provides explicit control over the additional false positives by bounding the probability that a correct decision is flipped.
It also exposes interpretable levers that can be tied to explicit service-level objectives in provenance-aware workflows.

The remainder of the paper is organized as follows. Sec.\,\ref{sec:related_work} surveys related work; Sec.\,\ref{sec:prelim} presents notation and preliminaries; Sec.\,\ref{sec:methodology} introduces the methodology and MILP formulation; Sec.\,\ref{sec:FPR-analysis} provides false-positive analyses under multiple threat models; Sec.\,\ref{sec:experiments} outlines our empirical results and deployment guidance. Sec.\,\ref{sec:discussion} discusses research implications, limitations, and future work.

\section{Related Work}
\label{sec:related_work}

\noindent\textbf{Information integrity.}
Information integrity is now widely recognized as a socio-technical challenge for information retrieval and content management systems \cite{ai2024reducing}. Platforms and institutions must decide at scale whether content is sufficiently authentic and trustworthy. Recent work spans consumer trust in AI-mediated content \cite{xylogiannopoulos2024chatgpt}, fake news detection \cite{fang2024nsep,luvembe2024cafodnn,peng2024contextual}, and statistical tests based on distributional regularities \cite{wang2025benfordipm}. Our contribution—a post hoc module for statistical watermarks that can be embedded in verification pipelines, respects operational false positive budgets, and supports streaming deployment and data-drift monitoring—fits naturally into these concerns.
Relative to the existing detection landscape, e.g., fake‑news models that fuse
local context and global signals \cite{fang2024nsep} and multimodal
co‑attention detectors \cite{luvembe2024cafodnn,peng2024contextual},
signature filtering plays the role of a \emph{domain‑agnostic pre‑filter} for the
textual channel that can be combined with complementary evidence such as network cues.

\medskip
\noindent\textbf{LLM watermarks.}
Watermarking has been adapted to LLMs for attributing
AI‑generated text \cite{liu2024survey,wu2025survey}.
The signature approach in this work is compatible with two mainstream families
of watermarking methods, referred to as the \textsc{Kgw} and \textsc{Exp}
families \cite{pan2024markllmopensourcetoolkitllm,pan2024waterseeker}: the
former biases the model's output distribution to embed signals
\cite{kirchenbauer2024watermarklargelanguagemodels,lee2024wrotecodewatermarkingcode,
kirchenbauer2024reliabilitywatermarkslargelanguage,
zhao2023provablerobustwatermarkingaigenerated}, while the latter guides token
selection without distorting the distribution
\cite{aaronson2022watermarking,christ2024undetectable,kuditipudi2024robust}.
\neww{Production-oriented watermarks like SynthID-Text \citep{dathathri2024scalable} instead avoid repeated bias and preserve text quality using repeated context masking, which skips watermarking and scoring when a previously used context window reappears in the text.}
From the managerial perspective, these methods serve as \emph{inline provenance markers} that downstream systems can verify post‑hoc.

\medskip
\noindent\textbf{Detection under challenging scenarios.}
Watermark detection is difficult when the watermark signal is weak, the text has low
variation, or the content is edited. Several \textit{detection-time} techniques tackle
these challenges without changing the embedded watermark. Entropy-based
detectors like \textsc{Ewd} reweight token contributions so that high-entropy
positions dominate the statistic, improving robustness on low-entropy
material \citep{lu2024entropybasedtextwatermarkingdetection}. \textsc{Sweet} instead extends
logit-modifying watermarks to code by suppressing low-entropy segments at
generation and detection time \citep{lee2024wrotecodewatermarkingcode}.
Window-based schemes such as \textsc{WinMax} replace the global score with the
maximum over sliding windows to recover signal after edits and dilution,
and \textsc{WaterSeeker} further localizes watermarked regions in long documents
\citep{kirchenbauer2024reliabilitywatermarkslargelanguage,pan2024waterseeker}.
Since these methods redesign the detection statistic, their false positive behavior must be
recalibrated for each watermark and operating condition.

By contrast, signature filtering removes selected tokens and reuses the original z-test of the underlying watermark scheme. This post-processing step does not reduce the baseline TPR and FPR, allowing us
to isolate and bound the incremental false positive risk caused solely by filtering.
Existing enhancement methods that change the scoring rule or mix embedding and detection do not preserve this structure, so they cannot
offer the same type of drop-in risk guarantees relative to a fixed baseline watermark test.

\medskip
\noindent\textbf{Optimization perspectives.}
Constraint‑based optimization appears at both the embedding and detection stages.
\citet{wouters2023optimizing} recast watermark scheduling as a bi-objective MILP that
balances detection power and text quality. \new{\citet{wang2025morphmark} likewise
frame \textsc{Kgw}-style watermark embedding as a multi-objective trade-off, adaptively adjusting watermark strength to better navigate the detectability-quality frontier.} \citet{li2024truncatedgof,li2025statistical} instead derive
Neyman-Pearson‑optimal tests via convex duality and formulate goodness‑of‑fit tests on the detection side to resist paraphrasing.
\citet{tsurheavywater} proposed an optimization framework for designing distortion-free watermarks in low-entropy regimes, jointly optimizing the watermarked next-token distribution and the detection score via a minimax objective. This connects watermark design to optimal transport and coding theory and provides tunable detection-distortion trade-offs in the embedding step.
Compared with these methods, we strengthen the \emph{existing} one‑proportion z‑test statistic using MILP, learning $n$‑gram filters
on representative data without changing the underlying watermark family. In this sense, our detection-time
optimization is orthogonal and combinable with embedding-time optimizations as mentioned above.
Finally, mixed‑integer programs can be used to recover secret partitions from limited samples
\citep{zhang2024stealing,reynolds2025breaking}. Because signature filtering relies only on observable token
statistics, it may still improve detection even if a green/red split is partially known. Analysis of this threat model is an interesting direction for future work.

\OMIT{
\section{Related Work}
\label{sec:related_work}

Watermarking has been adapted to LLMs as a mechanism for attributing AI-generated text \cite{liu2024survey,wu2024surveyllmgeneratedtextdetection}.
The signature approach in this work is compatible with two mainstream families of watermarking methods in the literature, referred to as the \textsc{Kgw} and \textsc{Exp} families \cite{pan2024markllmopensourcetoolkitllm,pan2024waterseeker}: the former biases the LLM's output distribution to embed watermark signals \cite{kirchenbauer2024watermarklargelanguagemodels,lee2024wrotecodewatermarkingcode,
kirchenbauer2024reliabilitywatermarkslargelanguage,
zhao2023provablerobustwatermarkingaigenerated}, while the latter guides the LLM's token generation without distorting the output distribution \cite{aaronson2022watermarking,christ2024undetectable,kuditipudi2024robustdistortionfreewatermarkslanguage}.


Watermark detection becomes challenging when the watermark signal is weak, the text variation is low, or the watermarked text has been edited.
A wide array of approaches aims to address these challenges without altering the embedding step
\cite{lu2024entropybasedtextwatermarkingdetection,lee2024wrotecodewatermarkingcode,kirchenbauer2024reliabilitywatermarkslargelanguage,pan2024waterseeker}.
Our signature method is orthogonal to these techniques: it collects statistically harmful $n$-grams and removes them before the z-score or p-value is computed, providing an additional boost in detection rate.

A watermark is expected to remain detectable after adversarial editing.
A wide array of robustness-strengthened watermarking schemes has been proposed
\cite{kirchenbauer2024watermarklargelanguagemodels,
      zhao2023provablerobustwatermarkingaigenerated,
      liu2024semanticinvariantrobustwatermark,
      he2024watermarkssurvivetranslationcrosslingual}.
For example, \textsc{Sir/x-Sir}~\cite{liu2024semanticinvariantrobustwatermark,he2024watermarkssurvivetranslationcrosslingual} aligns the red/green assignments semantically so
that paraphrases preserve the pattern across languages, albeit at the
cost of requiring model-internal embeddings at detection time.
By contrast, our signature filter works solely on the generated
text; it therefore complements semantic-aware schemes and improves
robustness even when only the plaintext output is available.


Constraint-based optimization has been applied to different stages of the watermark pipeline.
On the generation side, Wouters \cite{wouters2023optimizing} recasts watermark {embedding} as a
bi-objective MILP that balances detection power against {text quality}.
The green-red bias at each step becomes a decision variable, and the solver enumerates Pareto-optimal shift schedules to produce the watermark. Their generation-time schedule is complementary to our signature filter, which operates solely at detection time and can therefore adopt any precomputed schedule.
On the detection side, Li et al.~\cite{li2025statistical} derive a Neyman-Pearson-optimal z-score test by
solving a convex dual programming, whereas we strengthen the z-test statistics by MILP-based
$n$-gram filtering; the two techniques are orthogonal and can be combined for further gains.
Li et al.~\cite{li2024truncatedgof} address watermark robustness issues by formulating paraphrase-resilient
goodness-of-fit tests as an optimization problem. In contrast, our method improves
robustness by \emph{excluding} statistically fragile $n$-grams rather than redefining the test
itself. Finally, mixed-integer programs have been used offensively: \citet{zhang2024stealing} and \citet{reynolds2025breaking} show that the hidden green list (or secret key) can be recovered from only a handful of watermarked samples, enabling subsequent scrubbing or spoofing attacks.
Since our signature filter leverages observable token statistics rather than the secrecy of the green-red partition, it may still improve detection even when an adversary uncovers that partition.
Analysis of this threat model is an interesting direction for future work.
}

\section{Preliminaries}
\label{sec:prelim}
\noindent\textbf{Tokens and token types.}
A \emph{text} $T$ is a finite sequence of \emph{tokens}.
Each token in a text is assigned a \emph{token type} $t_{i}\in\mathcal V$ from a fixed finite \emph{vocabulary} $\mathcal V$.
For simplicity, we often identify $T$ of length $n$ with a sequence $\langle t_{1},\dots,t_{n} \rangle$ of token types.
Multiple tokens may share the same type.
For example, the text $\langle 1,2,1,2,3 \rangle$ contains five tokens $t_{1},\dots,t_{5}$ but only has three distinct token types: $t_1=t_3=1$, $t_2=t_4=2$, and $t_5=3$. 

\medskip
\noindent\textbf{Watermark embedding.}
Given a token sequence prefix \(\langle t_1, \dots, t_{i-1}\rangle \), the LLM generates the next token $t_i$ by computing a \emph{logit} \( l_i \in \mathbb{R}^{|\mathcal{V}|} \) and then sampling $t_i \in \mathcal V$ based on the probability distribution induced by the softmax of $l_i$.
A \textsc{Kgw}-style watermark \cite{kirchenbauer2024watermarklargelanguagemodels}
randomly splits $\mathcal V$ into a green list $\mathcal G_i$ and a red list $\mathcal R_i$ at each step $i$, such that $|\mathcal G_i| = \gamma\,|\mathcal V|$ with a green ratio $\gamma\in(0,1)$.
\textsc{Kgw} amplifies the sampling probability of green tokens with a \emph{selection bias} $\delta>0$.
A larger bias $\delta$ makes the watermark signal stronger and easier to detect, but may increase text distortion.

An \textsc{Exp}-style algorithm \cite{kuditipudi2024robust} guides token sampling through keyed pseudorandomness.
At each step $i$, \textsc{Exp} selects as the token $t_i$ a token satisfying
$
    t_i \in {\arg\max}_{y \in \mathcal V}~\bigl(l_i(y)/\theta  + G_i(y) \bigr),
$
where $\theta$ is the model temperature and $G_i(y)\!\sim\!\textrm{Gumbel}(0,1)$ denotes the (keyed) Gumbel noise \cite{fu2024gumbelsoft}.
\new{The model temperature $\theta$ controls sampling entropy and text diversity, which indirectly affects watermark detectability. When $\theta$ is small, the scaled logits $l_i(y)/\theta$ become more dominant. Token choices thus concentrate on a few high-confidence candidates, yielding a weaker accumulated watermark signal and making detection harder.
Conversely, a higher temperature strengthens watermark signals at the price of higher generation randomness and potential quality changes.}

\medskip
\noindent\textbf{Watermark detection.}
Given a candidate text $T$ of length $|T|=n$, the detector runs a hypothesis test to determine whether $T$ contains a watermark,
with the null hypothesis $\mathcal{H}_0$ claiming $T$ is unwatermarked.

A \textsc{Kgw} detector performs a z-test defined by
$
    Z_{\textrm{K}}(T) \coloneqq {(N_g-\gamma n)}/{\sqrt{\gamma(1-\gamma)n}},
$
where $N_g$ denotes the number of green tokens in $T$, and rejects $\mathcal H_0$ if and only if
$Z_{\textrm{K}}(T) \ge z_0$ for a prescribed threshold $z_0$.

\new{An \textsc{Exp} detector computes a per-token \emph{score} and the exact Gamma-tail p-value under $\mathcal H_0$ 
\cite{FernandezCTCF23WIFS}.
Specifically, at each position $i$, the detector generates a pseudorandom vector
$u_i \in (0,1)^{|\mathcal V|}$ (from the secret key and local context) and reads out the pseudorandom scalar
$
    R_i \coloneqq u_i[t_i].
$
It then defines the per-token \textsc{Exp} score
$
    c_i \coloneqq -\ln(1-R_i),
$
and the text-level statistic
$
    X \coloneqq \sum_{i=1}^{n} c_i.
$
Under $\mathcal H_0$, we have $R_i \sim \mathcal U[0,1]$. Hence
$c_i \sim \mathrm{Exp}(1)$ and $X \sim \Gamma(n,1)$. The \textsc{Exp} detector exploits the Gamma-tail p-value
$
    p_{\textrm{E}}(T) \coloneqq \mathbb P(X_n \ge X)
    = \Gamma(n, X)/\Gamma(n)
$
with $X_n \sim \Gamma(n,1)$,
and rejects $\mathcal H_0$ if and only if $p_{\textrm{E}}(T) < \alpha$
for a prescribed significance level $\alpha$.
Note that we can equivalently define a z-test
$Z_{\textrm{E}}(T) \coloneqq \Phi^{-1}(1 - p_{\textrm{E}}(T))$ and
$z \coloneqq \Phi^{-1}(1-\alpha)$,
such that $p_{\textrm{E}}(T)<\alpha \Longleftrightarrow Z_{\textrm{E}}(T) \ge z$.}

For both watermark schemes, detection can be uniformly formalized as a hypothesis test:

\begin{itemize}[leftmargin=*]
  \item $\mathcal{H}_0$ (\emph{Null hypothesis: $T$ is unwatermarked}).
  \new{For \textsc{Kgw}, $Z_{\textrm{K}}(T)$ is approximately $\mathcal{N}(0,1)$ for sufficiently large $|T|$ under $\mathcal{H}_0$.
  In contrast, \textsc{Exp} is an exact test for any $|T| \ge 1$: $p_{\textrm{E}}(T) \sim \mathcal U[0,1]$ and $Z_{\textrm{E}}(T) \sim \mathcal N(0,1)$ hold strictly under $\mathcal H_0$ without asymptotic reliance.}
  \item $\mathcal{H}_1$ (\emph{Alternative hypothesis: $T$ is watermarked}).
  Under $\mathcal{H}_1$, \textsc{Kgw} increases the expected green-token rate and inflates $Z_{\textrm{K}}(T)$ beyond zero. \textsc{Exp} biases the realized pseudorandom values $R_i$ toward larger values, making $X$ stochastically larger than its null distribution. This yields smaller $p_{\textrm{E}}(T)$ and larger $Z_{\textrm{E}}(T)$.
\end{itemize}



\noindent The baseline detector declares $T$ as watermarked if $Z(T)\ge z_0$, with $Z=Z_{\textrm{K}}$ for \textsc{Kgw} and $Z=Z_{\textrm{E}}$ for \textsc{Exp}. 
\neww{Practically, $z_0$ is the detector's decision threshold on the standardized evidence statistic and therefore controls the Type-I error budget. Increasing $z_0$ makes false positives rarer but also makes true watermarks harder to detect. For \textsc{Kgw}, when 
$Z_{\textrm{K}}(T)$ is approximately standard normal, a one-sided significance level $\alpha$ corresponds to $z_0 \approx \Phi^{-1}(1-\alpha)$. Our experiments adopt $z_0 = 4$ following prior work~\cite{kirchenbauer2024reliabilitywatermarkslargelanguage,pan2024markllmopensourcetoolkitllm}, which corresponds to a nominal one-sided significance level of about $3.17\times 10^{-5}$.
For Exp, $\alpha$ is specified directly and $z_0 \coloneqq \Phi^{-1}(1-\alpha)$ is the equivalent threshold for the z-test.}

\begin{figure*}[t]
\centering
\resizebox{\textwidth}{!}{%
\begin{tikzpicture}[
    auto,
    node distance = 0.4cm, 
    font=\footnotesize,    
    block/.style = {rectangle, draw, fill=blue!5, text width=1.3cm, text centered, rounded corners, minimum height=0.9cm, inner sep=2pt},
    process/.style = {rectangle, draw, fill=white, text width=1.3cm, text centered, minimum height=0.9cm, inner sep=2pt},
    decision/.style = {diamond, draw, fill=green!5, text width=1cm, text centered, inner sep=0pt, minimum height=1cm, aspect=1.3},
    data/.style = {draw, ellipse, fill=gray!10, text width=1.2cm, text centered, minimum height=0.7cm, inner sep=1pt},
    line/.style = {draw, -Latex, thick},
    dashedline/.style = {draw, -Latex, dashed, thick},
    outcome_w/.style = {data, fill=green!15, text width=2cm},
    outcome_n/.style = {data, fill=red!10, text width=2cm}
]

    
    \node [data] (corpus) {Train Data};

    \node [process, right=0.8cm of corpus] (scoring) {Color\\/Score};
    \node [process, right=0.8cm of scoring] (constraints) {Gen. Constr.};
    \node [block, right=0.8cm of constraints, fill=blue!10] (milp) {MILP Solver};

    \node [data, right=0.8cm of milp, fill=yellow!20] (sig) {Signature};

    
    \node [data, below=1.6cm of corpus] (input) {Input $T$};

    \node [block, right=0.8cm of input] (det1) {Detector};
    \node [decision, right=0.8cm of det1] (dec1) {$Z \ge z_0$};

    \node [block, right=0.8cm of dec1, text width=1.4cm] (filter) {Filter};
    \node [decision, right=0.8cm of filter, aspect=1.7, text width=1.3cm, inner sep=0pt] (declen) {$|T'| \!\ge\! n_0$};
    
    \node [block, right=0.9cm of declen] (det2) {Detector};
    \node [decision, right=0.8cm of det2] (dec2) {$Z'\!\ge z_0$};

    \node [outcome_w, above right=0.25cm and 2cm of dec2] (final_w) {``Watermarked''};
    \node [outcome_n, below right=0.25cm and 2cm of dec2] (final_n) {``Natural''};

    \path [line] (corpus) -- (scoring);
    \path [line] (scoring) -- (constraints);
    \path [line] (constraints) -- (milp);
    \path [line] (milp) -- (sig);

    \path [line] (input) -- (det1);
    \path [line] (det1) -- (dec1);

    \draw [line] (dec1.north) -- ++(0,0.3) -> (final_w.west);
    \node [font=\scriptsize, above right=0.55cm and 0.25cm of dec1] {Yes};

    \path [line] (dec1) -- node [pos=0.3, above, font=\scriptsize] {No} (filter);
    
    \path [dashedline] (sig.south) -- ++(0,-0.4) -| node [pos=0.6, below right, font=\scriptsize] {Inject} (filter.north);

    \path [line] (filter) -- (declen);
    
    \path [line] (declen) -- node [pos=0.3, above, font=\scriptsize] {Yes} (det2); 
    \path [line] (det2) -- (dec2);

    \draw [line] (dec2.east) -- ++(0.4,0) |- (final_w.west);
    \draw [line] (dec2.east) -- ++(0.4,0) |- (final_n.west);

    \draw [line] (declen.south) -- ++(0,-0.25) -> (final_n.west);
    \node [font=\scriptsize, below right=0.1cm and 0.45cm of declen] {No};
    
    \node [font=\scriptsize] at ($(dec2.east) + (0.7, 0.4)$) {Yes};
    \node [font=\scriptsize] at ($(dec2.east) + (0.7, -0.3)$) {No};

    \begin{scope}[on background layer]
        \node[fit=(corpus) (scoring) (sig), draw=gray, dashed, inner sep=0.3cm, rounded corners, fill=gray!5,
        label={[anchor=south west, xshift=.1cm, yshift=0cm, text=gray!80, font=\bfseries\small]north west:Phase 1: Offline Signature Generation}] (p1box) {};
        
        \node[fit=(input) (final_n) (final_w), draw=gray, inner sep=0.3cm, rounded corners,
        label={[anchor=south west, xshift=.1cm, yshift=-.55cm, text=gray!80, font=\bfseries\small]north west:Phase 2: Online Detection Pipeline}] (p2box) {};
    \end{scope}
\end{tikzpicture}}
\caption{\new{Flowchart of our two-stage watermark detection framework.}}
\label{fig:pineline}
\end{figure*}
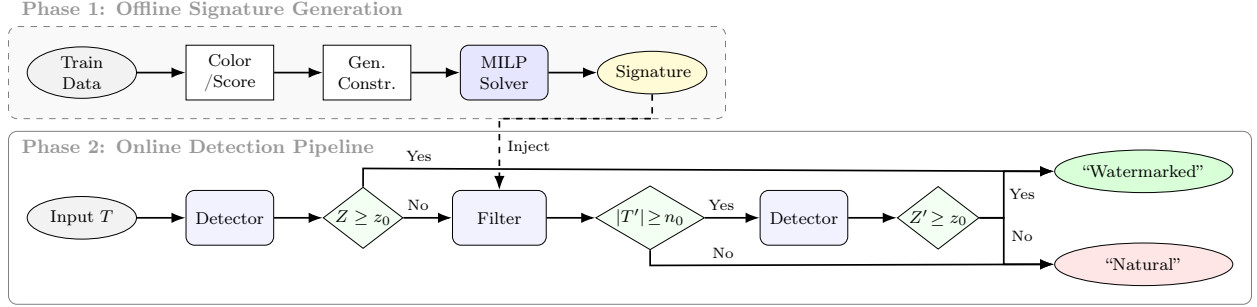

\medskip\noindent\textbf{Signature filtering.}
A \emph{signature} $S\subseteq\mathcal V$ is a set of \emph{token types} computed during or after text generation.
Let $Z(T)$ denote the scheme-specific test statistic. 
A detector equipped with a signature $S$ proceeds in two stages (Fig.\,\ref{fig:pineline}).
First, the detector performs the baseline test on the input text $T$, and declares ``watermarked'' if $Z(T)\ge z_0$.
Otherwise, it deletes from $T$ all tokens whose types belong to $S$, producing a residual text $T'$.
The detector then re-runs the same baseline test on $T'$, computes $Z' \coloneqq Z(T')$, and declares ``watermarked'' if $Z' \ge z_0$ and $|T'| \ge n_0$, a threshold for the filtered text length.
We set $n_0=30$ throughout this paper.\footnote{\new{For \textsc{Kgw}, $Z_{\textrm{K}}$ is a z-test that relies on a normal approximation of a binomial evidence count. Here, $n_0=30$ is a standard conservative rule of thumb to keep this approximation accurate (see e.g., \citet{hogg2015probability}, and also \ref{app:BE-to-remainders} for an explicit tail-bound characterization). For \textsc{Exp}, the Gamma-tail p-value remains exact under $\mathcal H_0$ for \textit{any} $|T'|\ge 1$, but we adopt the same $n_0$ as a minimum-evidence guardrail, keeping the two-stage procedure uniform across watermark schemes.}}

Since filtering is only applied when the baseline test fails, the only error it may introduce is to flip an otherwise correct ``$T$ is unwatermarked'' decision under $\mathcal H_0$. The error probability induced by the filter is therefore
$\Perror \coloneqq \mathbb{P}(Z'\ge z_0 \mid Z < z_0,\; \mathcal{H}_0)$.
This probability is a \emph{conditional} Type-I error: it measures how often filtering overturns a baseline acceptance of $\mathcal H_0$ when the underlying text is unwatermarked.

\medskip\noindent\textbf{Null calibration after filtering.}
For \textsc{Kgw}, under the standard coloring model, the hidden green indicators $\{\mathbf{1}\{t_i\in G_i\}\}_{i=1}^n$ are i.i.d.\ $\mathrm{Ber}(\gamma)$ conditioned on the secret key. Since our filter depends only on $T$ and does not access the key, restricting to the retained indices preserves this law. Thus, $Z_{\textrm{K}}(T')$ has the same null calibration as the original test $Z_{\textrm{K}}(T)$, and its $\mathcal N(0,1)$ approximation improves with the residual length $n' \coloneqq |T'|$. We therefore enforce $n' \ge n_0$ before acting on the post-filter statistic.
For \textsc{Exp}, each retained pseudorandom scalar $R_i=u_i[t_i]$ is $\mathcal U[0,1]$ under $\mathcal H_0$. The retained scores $c_i$'s remain i.i.d.\ $\mathrm{Exp}(1)$ and the post-filter sum
$X_{T'} \coloneqq \sum_{i=1}^{n'} c_i$ satisfies $X_{T'} \sim \Gamma(n',1)$, which implies $p_{\textrm{E}}(T')\sim \mathcal U[0,1]$ and $Z_{\textrm{E}}(T')\sim \mathcal N(0,1)$ in the Gamma model for any $n'\ge 1$. 
\neww{Thus, the statistic computed on the filtered natural text remains correctly calibrated for the baseline watermark test, provided the filter relies on observed content rather than on secret color or score assignments. On the other hand, giving the detector a second chance to reject the null hypothesis does introduce an additional false positive risk, which is captured by the conditional error probability $\Perror$ defined earlier. Sec.\:\ref{sec:FPR-analysis} explores this risk further under distributionally correlated and stronger adversarial settings, relaxing the assumption of null calibration after filtering.}

\section{Method}
\label{sec:methodology}

\subsection{MILP-based Signature Generation}
\label{sec:ILP-formulation}


Given a set of token types $\mathcal{V}$ and a training corpus $\mathcal{T}$, we seek a signature $S \subseteq \mathcal{V}$: if $S$ contains a token type $w$, then all tokens with the same type will be ignored when we compute the test statistic for a text.
For each $w \in \mathcal{V}$, we define binary decision variables $x_w$ such that $x_w=0$ means $w$ should be included in $S$.

\medskip
\noindent\new{\textbf{\textsc{Kgw}-style watermarks.}
Fix a text $T$. Let $c_{w}$ denote the number of tokens in $T$ that have type $w$, and
$g_{w}$ denote the number of its green occurrences.
Let $N$ and $N'$ be the number of tokens in $T$ before and after filtering,
and $N_g$ and $N_g'$ be the number of green tokens in $T$ before and after filtering. Then we can write
\begin{align}
N       = \sum\nolimits_{w \in \mathcal{V}} c_{w},\quad
N'      = \sum\nolimits_{w \in \mathcal{V}} c_{w} \, x_w,\quad
N_{g}   = \sum\nolimits_{w \in \mathcal{V}} g_{w},\quad
N_{g}'  = \sum\nolimits_{w \in \mathcal{V}} g_{w} \, x_w.
\end{align}
Let $p$ be a constant chosen so that ${N_{g}'} \ge p N'$ implies $Z_{\textrm K}(T) \ge z_0$ (see later). 
Note that
\begin{align}
{N_{g}'} - p N' \;=\; \sum\nolimits_{w \in \mathcal{V}}\, (g_{w} - p\, c_{w}) \, x_w    
\end{align}
is a linear function over the decision variables $\{x_w : w \in \mathcal V\}$.
}

\new{
For each text $T \in \mathcal T$, we define an indicator variable $y_T$ such that $y_T=1 \Longleftrightarrow {N_{g}'} \ge p N'\!$ for $T$.
We thus reduce signature generation to an optimization problem; any solution $\{x_w\}$ that maximizes $\sum_{T \in \mathcal T} y_T$ corresponds to a signature maximizing the post-filter detection rate of the corpus $\mathcal T$ and vice versa.
}

Finally, we specify the constant $p$ for each text $T$ such that ${N_{g}'} \ge p N'$ ensures that the residual text $T'$ after filtering is declared watermarked by the detector. For \textsc{Kgw}, $Z_{\textrm K}(T) \ge z_0$ holds if and only if 
\begin{align}\label{eq:linear}
 N_{g}' \;\;\ge\;\;
 \gamma N' + z_0 \sqrt{N'\,\gamma(1-\gamma)}
 \;=\;
 (\gamma + z_0 \sqrt{\gamma(1-\gamma)/N'}) \cdot N'.
\end{align}
When $N'$ is sufficiently large, that is, when
$
  N' \ge \gamma(1-\gamma)\cdot[z_0/(p-\gamma)]^2,
$
we can safely replace \eqref{eq:linear} with a sufficient condition ${N_{g}'} \ge p N'$.
For example, when $\gamma=0.5$ and $z_0 = 4$, we can set $p = 0.87$ to make ${N_{g}'} \ge p N'$ a criterion for detecting watermarks in texts with filtered length $N'\ge 30$. 
This linear approximation makes signature selection amenable to an MILP solver, with an objective to maximize the post-filter detection rate of the training corpus. \neww{Because the coefficients in the MILP are computed from green counts determined by the watermark key, the resulting signature is specific to that key and must be recomputed after key rotation.}

\medskip
\noindent\new{\textbf{\textsc{Exp}-style watermarks.}
We now describe how the same MILP idea applies to \textsc{Exp}.
Given a text $T$, the \textsc{Exp} detector assigns each token $t_i$ a score
$c_i \coloneqq -\ln(1-R_i)$ and uses the text-level statistic $X \coloneqq \sum_i c_i$ (Sec.\,\ref{sec:prelim}).
After applying a signature $S$ and retaining only token types with $x_w=1$, the residual text $T'$ has length $N'$ as above and accumulated score
$
X' \coloneqq \sum_{i:\,t_i\notin S} c_i.
$
Analogous to the \textsc{Kgw} case, we define  
\begin{align}
g_w \;\coloneqq\; \sum\nolimits_{i:\,t_i=w} c_i, \quad
X' = \sum\nolimits_{w\in\mathcal V} g_w\,x_w,\quad
X' - p N' = \sum\nolimits_{w\in\mathcal V} (g_w-p\,c_w)\,x_w.
\end{align}
\ref{app:aar-obj} shows how to derive a constant $p$ such that
\(
X' \ge p N'
\)
is sufficient for the \textsc{Exp} detector to declare $T'$ watermarked, i.e., it implies $p_{\textrm E}(T')<\alpha$. 
We can employ the same MILP structure as above for \textsc{Exp}, but replace the \textsc{Kgw} constraint
$N_g' \ge p N'$ with the \textsc{Exp} constraint $X' \ge p N'$.
Intuitively, $p$ induces a red-green split for \textsc{Exp}: tokens with scores $c_i < p$ behave as \emph{red} evidence, and signature filtering aims to remove such low-score tokens so that the average retained score per token exceeds the threshold $p$. 
}

\smallskip
It is worth noting that our formulation of optimal signature selection is computationally equivalent to the \emph{Maximum Feasible Subsystem (Max-FS) of 0--1 Linear Inequalities} \cite{amaldi1995complexity}. A Max-FS has a set of linear inequalities of the form $\sum_{j=1}^n a_{i,j} \, x_j \ge 0$; the task is to find a binary assignment to $x_1,\dots,x_n$ that maximizes the number of satisfied inequalities. Finding an optimal signature can be reduced to solving a Max-FS problem. The Max-FS literature has developed specialized solvers \cite{pfetsch2008branch} and fast relaxation heuristics (e.g., \citet{firouzeh2022faster}) that are highly effective in practice. Thus, signature generation can leverage the rich algorithmic toolkit for Max-FS to enhance practical scalability.

\subsection{Context-Sensitive Token Selection}
\label{sec:ngram-optimization}
A signature that amplifies watermark signals aims to reduce red evidence while preserving green evidence. However, treating a token type as either always kept or always deleted is often too coarse. In a document, the same type may appear as red in one position and green in another, so uniform deletion either leaves residual red noise or discards valuable green evidence.

To address this granularity gap, we extend type-level signatures to context-sensitive \emph{$n$-gram signatures}. Here, ``context-sensitive'' means that retention decisions are conditioned on a contiguous neighborhood of token types. Concretely, instead of assigning a single delete decision to each token type, the signature is parameterized by admissible $n$-grams, and token retention is induced by whether the $n$-grams that cover a token are admissible. This formulation can, for example, suppress a red-leaning local pattern when it occurs in a specific surrounding token-type configuration, while preserving other occurrences of the same token types in different local contexts. When $n=1$, this scheme degenerates to the original type-level (no-context) filter; increasing $n$ provides finer control over which recurring local fragments are removed. In our experiments, modest values of $n$ yield substantial detection gains in weak-signal settings.

\new{To illustrate, assume that each English word corresponds to one token type. Consider the sequence
\[
T=\langle \texttt{analysis},\ \texttt{of},\ \texttt{the},\ \texttt{data},\ \texttt{is},\ \texttt{of},\ \texttt{interest},\ \texttt{to},\ \texttt{the},\ \texttt{community}\rangle .
\]
Suppose the optimized 2-gram signature declares the 2-gram $(\texttt{of},\texttt{the})$ inadmissible
while keeping all other observed 2-grams admissible. Under our 2-gram retention rule
(a token is kept iff \textit{all} length-$2$ windows that cover it are admissible), the filtered text becomes
\[
T'=\langle \texttt{analysis},\ \texttt{data},\ \texttt{is},\ \texttt{of},\ \texttt{interest},\ \texttt{to},\ \texttt{the},\ \texttt{community}\rangle .
\]
In this example, each of the token types \texttt{of} and \texttt{the} appears twice, yet the filter removes only the
\texttt{of the} occurrence. By comparison, a type-level (1-gram) signature would have to delete \texttt{of} or \texttt{the} everywhere to eliminate \texttt{of the}, which would potentially discard useful green evidence.}

\medskip
\noindent\new{\textbf{MILP encoding.}
An $n$-gram signature can be viewed as selecting a set of \emph{admissible local contexts} (the $n$-grams),
and then retaining only those token occurrences whose surrounding contexts are admissible.
This leads to a standard incidence-based 0--1 formulation where binary variables choose which contexts are allowed, and constraints propagate these choices to token occurrences. We use two kinds of binary variables:
\begin{itemize}[itemsep=0pt, topsep=5pt]
  \item $x_{w}=1$ indicates that the token type $w \in \mathcal V$ is retained after filtering;
  \item $y_{\tau_1,\dots,\tau_n}=1$ indicates that the $n$-gram $(\tau_1,\dots,\tau_n)\in\mathcal V^{n}$
        is declared admissible by the signature.
\end{itemize}}
\new{
Concretely, consider the 2-gram case.
Any interior position $j$ (i.e., $2\le j\le m-1$ for $|T|=m$) is covered by two 2-grams,
so the retention rule is the Boolean conjunction
$
x_{t_j} = y_{t_{j-1},\,t_{j}} \wedge y_{t_{j},\,t_{j+1}}.
$
Boundary positions are handled by the 2-grams
$x_{t_1} = y_{t_{1},\,t_{2}}$ and $x_{t_m} = y_{t_{m-1},\,t_{m}}$.
These Boolean constraints can be encoded using the standard 0--1 linearization \cite{wolsey1998integer};
the objective function and corpus-level regularization remain the same as in Sec.\,\ref{sec:ILP-formulation}.
Generally, any feasible assignment to $\{y_{\tau_1,\dots,\tau_n}\}$ in this MILP defines an $n$-gram signature
$S := \{(\tau_1,\dots,\tau_n)\in\mathcal V^n : y_{\tau_1,\dots,\tau_n}=0\}$.
}





Note that a naive formulation of an $n$-gram signature introduces $O(k^{n})$ binary
decision variables, where $k$ is the number of distinct token types.
We mitigate this blow-up by considering only \emph{effective} $n$-grams,
i.e., the $n$-grams whose removal would discard strictly more red than green tokens,
since removing ineffective $n$-grams only weakens the watermark signal or leaves it unchanged.
This optimization drastically reduces the number of decision variables needed,
keeping the MILP within the reach of off-the-shelf solvers.

\OMIT{
\subsection{Pruning the Search Space}
A naive $n$-gram formulation introduces $O(N^{n})$ binary
variables $\{y_{\tau_1,\dots,\tau_n}\}$,
where $N$ is the number of distinct token types in the corpus.
We mitigate this blow-up by considering only the effective $n$-grams.
An $n$-gram $(\tau_1,\dots,\tau_n)$ is called \emph{effective} iff, in at
least one of its corpus occurrences, removing the $n$-gram would
discard strictly more red than green tokens.
Variables corresponding to ineffective $n$-grams are arguably not needed,
since they can only weaken the watermark signal or leave it unchanged.

The right plot of Fig.\,\ref{fig:workflow} shows that the growth of effective 2-grams is significantly slower than that of full 2-grams, keeping the MILP within reach of off-the-shelf solvers. In our evaluation, optimal 2-gram signatures can be found in seconds for corpora with hundreds of thousands of tokens.
}

\section{Asymptotic False-Positive Analysis}
\label{sec:FPR-analysis}

While signature filtering can boost the TPR, an overly aggressive filter can also raise the FPR by making an ordinary text look artificially ``watermark-heavy'' after deletion. The mechanism behind this risk depends on the underlying watermark family. In \textsc{Kgw}-style schemes, the evidence is the (hidden) green/red coloring of tokens, and deleting tokens can change the observed green fraction in ways that are potentially harmful for adversarial or correlated data. \new{In \textsc{Exp}-style schemes, the evidence is a keyed pseudorandom score per realized token, and a filter that depends only on the observed token sequence cannot cherry-pick unusually large \textsc{Exp} scores. Consequently, filtering is intrinsically more robust for \textsc{Exp} than for \textsc{Kgw}.} 

In Secs~\ref{sec:model-A}--\ref{sec:model-D}, we analyze \textsc{Kgw} under three threat models: (i) a signature-aware but color-blind attacker, (ii) a signature-aware and color-adaptive attacker, and (iii) distributionally correlated signatures without an active attacker. These results show when \textsc{Kgw} filtering is guaranteed to be safe: deleting a number of tokens that grows linearly with the text length is essentially harmless under the standard coloring assumption, but adversarial or correlated deletions on the order of $\sqrt{n}$ can potentially degrade detection in weaker conditions. \new{We summarize the corresponding false-positive guarantees for \textsc{Exp} in Sec.\,\ref{sec:model-E}.}

Throughout the analyses, we fix a z-score threshold $z>0$ and a tolerance error level $\varepsilon\in(0,1)$.
For the \textsc{Kgw}-specific bounds, we additionally fix a green ratio $\gamma \in (0,1)$.
For \textsc{Exp}, we write $z \coloneqq \Phi^{-1}(1-\alpha)$ for the detector's significance level $\alpha$.
The error probability of a filter is defined by $\Perror \coloneqq \mathbb{P}(Z'\ge z \mid Z < z, \mathcal{H}_0)$,
as discussed in Sec.\,\ref{sec:prelim}. Full proofs of the theorems can be found in the appendices.



\subsection{Signature-Aware but Color-Blind Attacker}
\label{sec:model-A}

Under this threat model, the attacker has observed the entire signature before crafting the text.
She may choose any tokens in the text, but has no control over the color of each token. 
This ``color-blind'' assumption aligns with the standard setting of all \textsc{Kgw}-style watermarks, where the hash seed that partitions the vocabulary is opaque to external users. In this setting, a linear deletion budget is sufficient to make the signature filter provably resilient to even a fully signature-aware adversary.


\begin{theorem}\label{thm:finite-sig-bound}
  For every sufficiently large text length~$n$,
  one can compute a bound $s_{\mathrm{safe}} = \Theta(n)$ such that
  any signature filter that deletes at most $s_{\mathrm{safe}}$ tokens
  guarantees $\Perror \le \varepsilon$.
\end{theorem}

Theorem~\ref{thm:finite-sig-bound} shows that signature filtering behaves as a bounded perturbation of
the baseline z-test whenever the filter operates within this linear deletion regime.
This result is a conservative worst-case estimate: in practice, it offers a simple rule of thumb for configuring
deletion budgets, which can be refined by empirical calibration in specific deployments.




\subsection{Signature-Aware and Color-Adaptive Attacker}
\label{sec:model-C}
In this scenario, the attacker not only knows the deployed signature but can also freely choose both the tokens and their colors when crafting a text. This setting corresponds to a fully compromised watermark, e.g., the attacker has effectively learned the secret partition of the watermark \cite{reynolds2025breaking}.
%
%
%
Theorem \ref{thm:sqrt-deterministic} characterizes how much adversarial editing is needed to overturn a robustly correct unwatermarked decision under full color control. Specifically, when the pre-filter z-score lies below the threshold by a fixed margin, a deletion budget of order $\Theta(\sqrt n)$ is both sufficient and necessary to \emph{deterministically} flip the decision. 
This result holds independently of the coloring assumption behind the watermark.

\begin{theorem}
\label{thm:sqrt-deterministic}
Fix a constant $\eta>0$. For every sufficiently large text length $n$,
a color-adaptive attacker who knows the signature can construct a text with
$Z \le z-\eta$ such that deleting $O(\sqrt{n})$ tokens ensures $\Perror = 1$. 
Conversely, for any text with $Z \le z-\eta$, no
deterministic flip is possible with $o(\sqrt n)$ deletions.
\end{theorem}


We note that a color-adaptive attacker already has enough power to break the watermark without filtering. In such case, the underlying watermark has lost its secrecy, and no detection‑time module can repair it. 
If color‑adaptive threats are considered plausible, the appropriate response is to refresh the watermark key or combine watermarking with additional provenance signals.



\subsection{Distributionally Correlated Signatures}
\label{sec:model-D}
In practice, a signature might be deployed on the same distribution that informed the signature selection, e.g.,
when the training corpus shares low-entropy phrases such as boilerplate fragments or topical keywords with a benign user text. 
Hence, the red tokens removed by the filter are no longer a random sample, but a biased subset of the text.
Theorem \ref{thm:square-law} formalizes the \emph{worst case} in this setting: $\sqrt{n}$-scale adversarial deletions is necessary and sufficient to flip a decision {with an arbitrarily high probability} given admissible parameters.

\OMIT{
\begin{theorem}[Safe deletion with controlled data dependency]
\label{thm:delta-bounded-intrinsic}
Fix parameters
\[
    0<\gamma<1,\quad
    0<\delta<\min\{\gamma,1-\gamma\},\quad
    0<\varepsilon<1.
\]
Assume that
\[
z\ge z_\varepsilon
        \coloneqq\inf\Bigl\{z>0:\frac{1-\Phi(z)}{\Phi(3z/4)}\le\frac{\varepsilon}{4}\Bigr\}.
\]
\noindent
\textbf{Per-token probabilities and variance.}\;
For an arbitrary choice of
$p_i\in[\gamma-\delta,\gamma+\delta]$ let
$\tilde\sigma_{\!n}^{2}\coloneqq\sum_{i=1}^{n}p_i(1-p_i)$.

\medskip
\noindent
\textbf{Deletion budget.}\;
Denote by $\rho_\ast=\rho_\ast(z,\varepsilon)$ the correlation threshold
from Theorem\,4, and set
\[
   s_\star(n)
      \coloneqq\Bigl\lfloor
             \min\!\Bigl\{
                   \frac{z\,\tilde\sigma_{\!n}}{4\delta},
                   \;
                   (1-\rho_\ast^{\,2})n
             \Bigr\}
         \Bigr\rfloor_+ .
\]
\noindent
\textbf{Guarantee.}\;
Let an attacker (i) choose any signature
$S\subseteq[n]$ of size $s\le s_\star(n)$,
(ii) choose $p_i \in [\gamma-\delta,\gamma+\delta]$ for any $i \in n$,
and choose $p_i = \gamma$ for any $i \notin S$,
(iii) color tokens
independently, and (iv) delete the $s$ signature tokens.
Then for all sufficiently large $n$, it holds that
\[
   \mathbb{P}\bigl(Z'\ge z \mid Z<z\bigr)\;\le\;\varepsilon,
\qquad
   Z\coloneqq\frac{N_g-\gamma n}{\tilde\sigma_{\!n}},
\quad Z'\coloneqq\frac{N_g'-\gamma(n-s)}{\tilde\sigma_{n-s}} .
\]
where $N_g$ and $N_g'$ are the numbers of tokens before and after filtering.
\end{theorem}
\begin{proof}
Define $c_z \coloneqq \tfrac12\Phi(3z/4)$.

\emph{Step 1 (bias control).}
With $\psi\coloneqq\sum_{i\in S}(p_i-\gamma)$ we have $|\psi|\le\delta s\le z\tilde\sigma_{\!n}/4$.
Write $X\coloneqq\sum_{i=1}^{n}(1_{\{\text{token $t_i$ is green}\}}-p_i)$ with $\mathbb E[X]=0$.
Hence $Z=\psi/\tilde\sigma_{\!n}+X/\tilde\sigma_{\!n}$ satisfies
$Z<z\Rightarrow X/\tilde\sigma_{\!n}\le3z/4$.

\emph{Step 2 (lower bound for $\mathbb{P}(Z<z)$).}
Apply the non-identical Berry-Esseen inequality (constant 0.56) to
$X/\tilde\sigma_{\!n}$:
\[
   \mathbb{P}(Z<z)\;\ge\;\Phi(3z/4)-0.56/\tilde\sigma_{\!n}\;\ge\;c_z
\]
for sufficiently large $n$.

\emph{Step 3 (upper tail of $Z'$).}
Since $s\le(1-\rho_\ast^{\,2})n$, one has
$\tilde\sigma_{n-s}\ge\rho_\ast\tilde\sigma_{\!n}$.
A second Berry-Esseen application gives
\[
   \mathbb{P}(Z'\ge z)
     \;\le\;1-\Phi(z)+0.56/(\rho_\ast\tilde\sigma_{\!n})
     \;\le\;1-\Phi(z)+\tfrac12\varepsilon c_z
\]
for sufficiently large $n$.

\emph{Step 4 (conditioning).}
Since $\mathbb{P}(Z'\ge z,\,Z<z)\le\mathbb{P}(Z'\ge z)$
dividing the r.h.s.~by $\mathbb{P}(Z<z)\ge c_z$ and using $z\ge z_\varepsilon$ gives
$\mathbb{P}(Z'\ge z\mid Z<z)\le\varepsilon$.
\end{proof}
\begin{theorem}[Safe deletion under $\delta$-bounded data dependence]
\label{thm:delta-bounded-intrinsic}
Fix parameters
\[
    0<\gamma<1, \qquad
    0<\delta<\min\{\gamma,1-\gamma\}, \qquad
    0<\varepsilon<1 .
\]
Let
\[
    z_\varepsilon
        \coloneqq\inf\Bigl\{z>0:\,
              \frac{1-\Phi(z)}{\Phi(3z/4)}\le\frac{\varepsilon}{4}
           \Bigr\},
    \qquad
    z\;\ge\;z_\varepsilon .
\]

\paragraph{Per-token probabilities and variance.}
For any choice of
$p_i\in[\gamma-\delta,\gamma+\delta]$ $(1\le i\le n)$
define the (heteroscedastic) variance
\[
    \tilde\sigma_{\!n}^{2}\;\coloneqq\;\sum_{i=1}^{n}p_i(1-p_i).
\]

\paragraph{Deletion budget.}
Let $\rho_\ast=\rho_\ast(z,\varepsilon)\in(0,1)$ be
the correlation threshold from Theorem\,4 and set
\[
   s_\star(n)
      \coloneqq\Bigl\lfloor
             \min\!\Bigl\{
                   \tfrac{z\,\tilde\sigma_{\!n}}{4\delta},
                   \;(1-\rho_\ast^{\,2})n
             \Bigr\}
         \Bigr\rfloor_+
   \quad\bigl(\text{take }s_\star=0\text{ if the bracket is negative}\bigr).
\]

\paragraph{Guarantee.}
Let an attacker  
(i) choose any signature $S\subseteq[n]$ with $|S|=s\le s_\star(n)$;  
(ii) set $\,p_i\in[\gamma-\delta,\gamma+\delta]$ for $i\in S$ and $p_i=\gamma$ for $i\notin S$;  
(iii) color tokens independently according to those $p_i$;  
(iv) delete the $s$ signature tokens.

Then, for all sufficiently large $n$, the flip probability satisfies
\[
   \mathbb{P}\bigl(Z'\ge z \;\bigm|\; Z<z\bigr)\;\le\;\varepsilon,
   \qquad
   Z \coloneqq \frac{N_g-\gamma n}{\tilde\sigma_{\!n}},
   \quad
   Z' \coloneqq \frac{N_g' -\gamma(n-s)}{\tilde\sigma_{n-s}},
\]
where $N_g$ (resp.\ $N_g'$) is the number of green tokens before (resp.\ after) deletion.
\end{theorem}
\begin{proof}
Throughout, write
\[                    
     X
     \coloneqq\sum_{i=1}^{n}\bigl(\mathbf 1_{\{\text{token $t_i$ green}\}} - p_i\bigr),
     \qquad
     \mathbb{E}X = 0,
\]
so that $N_g = \sum_{i=1}^{n}p_i + X$.

\smallskip\noindent
\textbf{Step 1 (bias control).}  
The \emph{bias term}
$
   \psi \coloneqq \sum_{i\in S}(p_i-\gamma)
$
satisfies $|\psi|\le\delta s\le z\tilde\sigma_{\!n}/4$ by the
definition of $s_\star$.
Hence
\[
      Z\;=\;\frac{\psi}{\tilde\sigma_{\!n}}
            +\frac{X}{\tilde\sigma_{\!n}}
      \quad\Longrightarrow\quad
      \{Z<z\}\subseteq\Bigl\{\frac{X}{\tilde\sigma_{\!n}}\le\tfrac34 z\Bigr\}.
\]
(The algebra is just $z>\psi/\tilde\sigma_{\!n}+X/\tilde\sigma_{\!n}$
with $|\psi|/\tilde\sigma_{\!n}\le z/4$.)

\smallskip\noindent
\textbf{Step 2 (lower bound for $\mathbb{P}(Z<z)$).}  
The summands of $X$ are \emph{independent, mean $0$,
bounded by $1$, and with finite third absolute moment}  
($\lvert\mathbf 1-p_i\rvert^3\le1$).
Therefore the triangular-array Berry-Esseen theorem of
Bentkus \cite[Theorem\,2.1]{Bentkus2005}
applies\footnote{%
It requires only independence, zero mean, $\sum\sigma_i^2=1$
after normalisation, and finite third absolute moments.}.
With the universal constant $C=0.56$ from that theorem we obtain
\[
   \mathbb{P}\Bigl(\tfrac{X}{\tilde\sigma_{\!n}}\le\tfrac34 z\Bigr)
      \;\ge\;
      \Phi(3z/4)-\frac{0.56}{\tilde\sigma_{\!n}}
      \;\;=\;:
      c_z - \frac{0.56}{\tilde\sigma_{\!n}},
      \quad
      c_z\coloneqq\tfrac12\Phi(3z/4).
\]
Because $\tilde\sigma_{\!n}= \Omega(\sqrt n)$,
the subtracted term is $o(1)$ and is $<c_z/2$ for all $n$ large enough,
yielding
\[
      \mathbb{P}(Z<z)\;\ge\; \tfrac12 c_z
      \qquad\text{for large }n.
\]

\smallskip\noindent
\textbf{Step 3 (upper bound for $\mathbb{P}(Z'\ge z)$).}  
After deleting $s$ tokens we have a new triangular array
of length $n-s$.  
Its variance satisfies $\tilde\sigma_{n-s}\ge\rho_\ast\tilde\sigma_{\!n}$
by the choice of $s_\star$.  
Applying Bentkus' bound \emph{again} (now to a \emph{sum of i.i.d.\ Bernoulli's}
because all remaining $p_i=\gamma$ on non-signature positions) gives
\[
   \mathbb{P}(Z'\ge z)
      \;\le\;
      1-\Phi(z)\;+\;\frac{0.56}{\rho_\ast\tilde\sigma_{\!n}}.
\]
For large $n$ the last fraction is $\le\frac12\varepsilon c_z$
by the same $\Omega(\sqrt n)$ argument.

\smallskip\noindent
\textbf{Step 4 (conditioning).}
Now
\[
   \mathbb{P}(Z' \ge z\mid Z<z)
        \;=\;
        \frac{\mathbb{P}(Z' \ge z,\;Z<z)}{\mathbb{P}(Z<z)}
        \;\le\;
        \frac{\mathbb{P}(Z'\ge z)}
             {\frac12 c_z}
        \;\le\;
        \frac{2\bigl[1-\Phi(z)+\frac12\varepsilon c_z\bigr]}{c_z}.
\]
Because $z\ge z_\varepsilon$ implies
$(1-\Phi(z))/c_z\le\varepsilon/2$, the right-hand side
is $\le\varepsilon$, completing the proof.
\end{proof}

\begin{theorem}[Safe deletion with $\sqrt n$-budget]
\label{thm:delta-bounded-intrinsic-v2}
Fix parameters
\[
    0<\gamma<1,\quad
    0<\delta<\min\{\gamma,1-\gamma\},\quad
    0<\varepsilon<1,
\]
and let
$
      z_\varepsilon\coloneqq\inf\{z>0:(1-\Phi(z))/\Phi(3z/4)\le\varepsilon/4\}
$
with $z\ge z_\varepsilon$.

\paragraph{Per-token probabilities.}
Choose arbitrary
$p_i\in[\gamma-\delta,\gamma+\delta]$ $(1\le i\le n)$ and let
$
      \tilde\sigma_{n}^{2}\coloneqq\sum_{i=1}^{n}p_i(1-p_i).
$

\paragraph{Deletion budget.}
Define
\[
   s_\star(n)\coloneqq\Bigl\lfloor\frac{z\,\tilde\sigma_N}{4\delta}\Bigr\rfloor_+ .
\]

\paragraph{Guarantee.}
For every sufficiently large $n$ and every signature
$S\subseteq[n]$ with $|S|=s\le s_\star(n)$,
the flip probability satisfies
\[
   \mathbb{P}(Z'\ge z\mid Z<z)\;\le\;\varepsilon,
\]
where
$
   Z=(N_g-\gamma n)/\tilde\sigma_N
$
and
$
   Z'=(N_g'-\gamma(n-s))/\tilde\sigma_{n-s}.
$
\end{theorem}
}
\begin{theorem}
\label{thm:square-law}
Suppose that the attacker can choose the colors of the removed tokens with a deletion budget $s \ge 0$. 
We can compute constants $c = c(z,\gamma)$ and $z_\varepsilon$
such that for all sufficiently large text length $n$,
(i) if $s\le \lfloor c\sqrt n\rfloor$ and $z \ge z_{\varepsilon}$, then it holds that $\Perror \le \varepsilon$;
(ii) if $\varepsilon \in (0, \tfrac{1}{2})$, then there exists a constant $c_{\mathrm{flip}} > c$ such that $s \ge \lceil c_{\mathrm{flip}}\sqrt n\rceil$ implies $\Perror \ge 1-\varepsilon$.
\end{theorem}


Theorem~\ref{thm:square-law} describes a worst-case scenario that principled signature designs should
avoid. This scenario occurs when a signature is too closely tuned to the deployment data and becomes
strongly correlated with natural text. Existing entropy- and semantics-aware methods can reduce this
correlation by down-weighting predictable or stale patterns~\cite{lee2024wrotecodewatermarkingcode,
lu2024entropybasedtextwatermarkingdetection,he2024watermarkssurvivetranslationcrosslingual}.
Because low-entropy phrases tend to appear in both
natural and watermarked text~\cite{kirchenbauer2024watermarklargelanguagemodels}, penalizing such phrases
during signature learning or filtering steers the system toward the linear-safe regime of
Theorem\,\ref{thm:finite-sig-bound}. In effect, we give up a small amount of best-case TPR to gain a lower
worst-case FPR. This trade-off is observed in our empirical comparison of \textsc{Kgw} and \textsc{Sweet}
(which ignores low-entropy tokens in the z-test) in Secs.\,\ref{sec:exp-accuracy}--\ref{sec:scalibility}.

\subsection{Signature Filtering for \textsc{Unigram} Watermarks}
\label{sec:unigram-colorblind}

In this subsection, we replace \textsc{Kgw}'s stepwise coloring scheme with the \emph{unigram coloring assumption} by \textsc{Unigram} \cite{zhao2023provablerobustwatermarkingaigenerated}, where each token type $\tau\in\mathcal V$ receives a single random color $C_\tau\sim\mathrm{Ber}(\gamma)$ that is reused by all of its occurrences in the text.
Formally, let $\ell$ be the number of distinct token types in a natural text, and $\{m_\tau\}_{\tau\in\mathcal V}$ be the pre-filter type multiplicities, namely, $m_\tau$ is the number of tokens of type $\tau\in\mathcal V$ in the text.
Write $n=\sum_{\tau \in \mathcal V} m_\tau$, $Q=\sum_{\tau \in \mathcal V} m_\tau^2$, $r={Q}/{n}$, and define their post-filter analogues $n'$, $Q'$, $r'$. Note that under unigram coloring, a filter either keeps or removes each type in full (so $m'_\tau\in\{0,m_\tau\}$).

\begin{theorem} 
\label{thm:unigram-budget}
Assume that no single token type carries a non-vanishing
fraction of the text, e.g., $\max_\tau m_\tau/n \to 0$ as $\ell\to\infty$.
%
Then there exists a constant $\beta = \beta(\varepsilon;z,\gamma,r,r')\in(0,1)$ such that for all $\ell$ sufficiently large, $Q' \ge \beta Q$ implies $\mathbb P(Z'\ge z \mid Z<z) \le \varepsilon$.
\end{theorem}

Theorem\,\ref{thm:unigram-budget} delineates the \emph{type-level} safe deletion budget for a signature. Intuitively, it shows that the right ``sample size'' for unigram coloring is the \emph{type-mass} $Q=\sum_\tau m_\tau^2$: the correlation between the pre-filter and post-filter texts is $\rho=\sqrt{Q'/Q}$, and it is sufficient to keep a constant fraction of $Q$.
By contrast, Theorem\,\ref{thm:finite-sig-bound} provides a safe budget for removable tokens: keeping a constant fraction of tokens $n'/n\ge c$ suffices to control the false-positive risk with a pre/post correlation $\rho=\sqrt{n'/n}$ under \textsc{Kgw} coloring.
In intuition, $Q/n^2=\sum_\tau (m_\tau/n)^2$ measures how concentrated repetition is, and a safe budget must preserve the mass carried by frequent types. Technically, both theorems arise from the same bivariate normal approximation with a vanishing remainder. They differ only in the correlation parameter, namely $\sqrt{n'/n}$ at the token level versus $\sqrt{Q'/Q}$ at the type level, and thus the natural linear budgets are in $n$ versus in $Q$.

\subsection{Signature Filtering for \textsc{Exp} Watermarks}
\label{sec:model-E}

\new{
The preceding subsections focus on \textsc{Kgw} and \textsc{Unigram}, where the post-filter z-score can change substantially because the filter removes a non-random subset of evidence. For \textsc{Exp}, watermarking is typically more robust. Recall that, under $\mathcal{H}_0$, \textsc{Exp} assigns each realized token a keyed pseudorandom scalar $R_i=u_i[t_i]\sim \mathcal U[0,1]$ and a score $c_i=-\ln(1-R_i)\sim \mathrm{Exp}(1)$, and it computes a p-value against a threshold $\alpha$ of significance level.
Define a ``score secrecy'' condition that (i) the signature depends only on the observed tokens and does not access these hidden random values, and (ii) the per-position vectors $u_i$ and realized scalars $R_i$ are not manipulatable or predictable by the adversary. When this condition holds, we can derive a universal worst-case risk bound that does not require a safe deletion budget in the sense of Sec.\:\ref{sec:model-A}.}

\begin{theorem}
\label{thm:exp-universal}
\new{Under the score secrecy assumption, $\Perror \le \alpha/(1-\alpha)$ holds for any signature filter.}
\end{theorem}

\new{
Intuitively, for an \textsc{Exp} detector under score secrecy, each retained token in the natural text contributes a fresh random value drawn from the same null distribution, and the test explicitly recalibrates to the retained length $n'$. Deleting tokens based on their \emph{content} is therefore like discarding some draws without seeing their random values: it does not allow the filter to systematically increase the length-normalized accumulated score. 
Below, we briefly discuss the \textsc{Exp} analogues of the threat models we have considered for \textsc{Kgw}:}

\begin{itemize}
    \item \new{\emph{Signature-aware but score-blind attacker.}
    Theorem~\ref{thm:exp-universal} holds when the attacker can fully adapt the text to the deployed signature, but
    cannot infer or influence the hidden \textsc{Exp} pseudorandom scores.}

    \item \new{\emph{Distributionally correlated signatures.}
    Even if the signature is trained on a corpus that strongly overlaps with the deployment distribution, the deletion decisions still depend only on token patterns. Under score secrecy, this does not bias the \textsc{Exp} scores for a natural text, so the same theorem holds.}

    \item \new{\emph{Score-adaptive attacker.} If the attacker can predict the per-token scores \(c_i\) and craft a text such that the signature deletes only tokens with near-zero scores, the same scaling in Theorem \ref{thm:sqrt-deterministic} applies: \(\Theta(\sqrt{n})\) deletions are sufficient and necessary to deterministically flip a baseline decision with \(Z\le z-\eta\).} 
\end{itemize}

\new{In practice, $\alpha$ is typically set to $10^{-4}$ \cite{pan2024markllmopensourcetoolkitllm}, making $\Perror$ effectively negligible by Theorem~\ref{thm:exp-universal}. This robustness result is consistent with our empirical findings: in all evaluated settings, applying signature filtering to \textsc{Exp} introduced no observable increase in false positive rate ($\le 0.1\%$) for the baseline detector.}

\smallskip\textit{\neww{Clarification on scale.}} \neww{We note that the false-positive rate analyzed in this section is a per-text quantity. The asymptotic parameter is the token length $n$ of the examined text, not the number of texts in a corpus. Hence, for a fixed signature filter and decision threshold, applying the detector to more documents does not by itself alter the per-text Type-I error guarantees. Corpus-level deployment, including the trade-off between TPR and FPR induced by deploying multiple signatures, will be discussed in Sec.\:\ref{sec:scalibility}.}

\OMIT{
Theorem\,\ref{thm:square-law} shows that distributional overlap can invalidate the linear safe budget asserted by Theorem\,\ref{thm:finite-sig-bound}.
The literature has offered many remedies to restore data independence, e.g., using entropy- or semantics-aware detectors
\cite{lee2024wrotecodewatermarkingcode,lu2024entropybasedtextwatermarkingdetection,liu2024semanticinvariantrobustwatermark,he2024watermarkssurvivetranslationcrosslingual}, which systematically attenuate correlations by down-weighting predictable tokens, preventing stale, high-overlap $n$-grams from dominating the test statistic.
These strategies help reduce the dependency between signature and human texts in practice, restoring the linear safety margin provided by Theorem\,\ref{thm:finite-sig-bound}.
Specifically, since low-entropy $n$-grams tend to appear in both watermarked and natural texts \cite{kirchenbauer2024watermarklargelanguagemodels,zhao2023provablerobustwatermarkingaigenerated}, one can penalize low-entropy $n$-grams during signature selection to reduce potential data correlation, which effectively trades the signature's best-case TPR for a lower worst-case FPR. This trade-off is observed in our empirical comparison on the performance of \textsc{Kgw} vs \textsc{Sweet}, see Sec.\,\ref{sec:experiments}.
}

\section{Evaluation}
\label{sec:experiments}

We evaluate our method on high-entropy natural language and low-entropy code from
standard benchmarks for watermark evaluation. For natural-language documents, we use \textsc{C4}~\cite{raffel2020exploring}. \new{For low-entropy code snippets, we use \textsc{Mbpp}~\cite{austin2021program}, \textsc{HumanEval}~\cite{chen2021evaluating}, and \textsc{Code-Search-Net}~\cite{husain2019codesearchnet}.}
Our experiments follow the standard setup~\cite{kirchenbauer2024watermarklargelanguagemodels,pan2024markllmopensourcetoolkitllm}: we tokenize each text, use the first 30 tokens as a prompt, and the next 200 tokens as the unwatermarked continuation. A watermark algorithm then generates watermarked 200-token continuations for every prompt, yielding parallel collections of watermarked and natural texts.
%

We compare four watermark schemes \textsc{Kgw}, \textsc{Sweet}, \textsc{Unigram}, \textsc{Exp}
\cite{kirchenbauer2024watermarklargelanguagemodels,lee2024wrotecodewatermarkingcode,zhao2023provablerobustwatermarkingaigenerated,aaronson2022watermarking}
on six LLMs: Opt-1.3b, Opt-6.7b, Llama2-13b, Llama3.1-8b, Qwen2.5-14b, Phi-3-Medium-14b. Watermarks are embedded with the MarkLLM toolkit~\cite{pan2024markllmopensourcetoolkitllm}, signatures are optimized with the Gurobi solver~\cite{gurobi}, and all experiments run on a computer with an Intel i9-14900 CPU, an RTX-5080 GPU, and 64 GB of RAM.

\subsection{Detection Capability of Signature Filters}
\label{sec:exp-accuracy}

\begin{figure*}[ht]
    \centering
    \includegraphics[width=\textwidth]{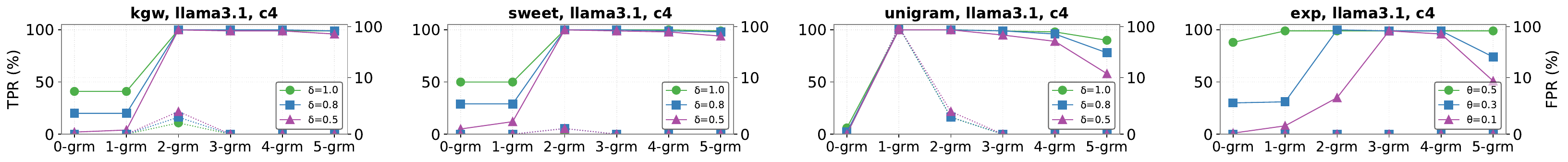}
    \includegraphics[width=\textwidth]{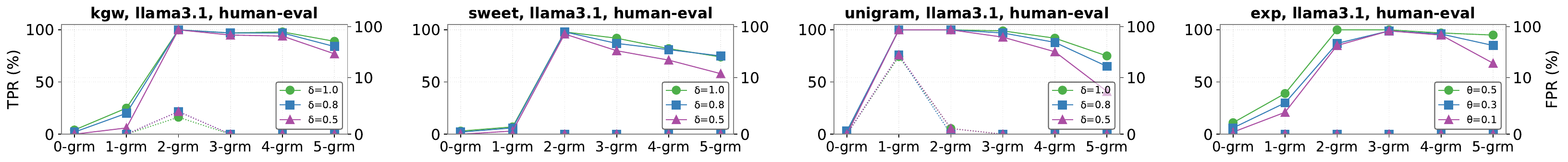}
    \includegraphics[width=\textwidth]{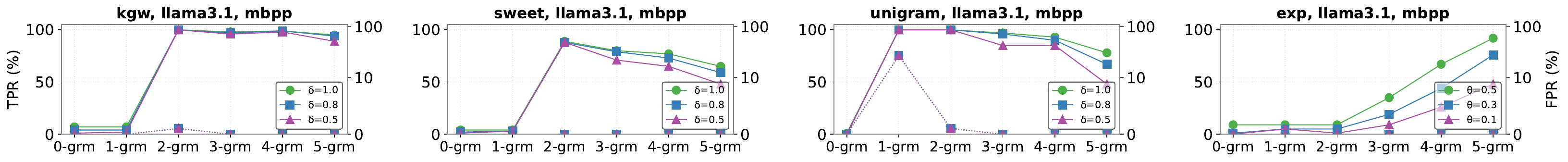}
    \caption{    
    \new{The TPR/FPR of $n$-gram signature filtering for Llama3.1-8b on C4 (top), \textsc{Mbpp} (middle), and \textsc{HumanEval} (bottom). Columns correspond to Kgw, Sweet, Unigram, and Exp. TPR (solid curves) is shown on the left y-axis in linear scale, while FPR (dashed curves) is shown on the right y-axis in log scale. The x-axis reports the signature order, where 0-gram corresponds to the baseline detector without filtering. The MILP solver timeout is 60 seconds per signature.} 
    }
    \label{fig:data1000}
\end{figure*}

\begin{figure*}[ht]
    \centering
    \includegraphics[width=\textwidth]{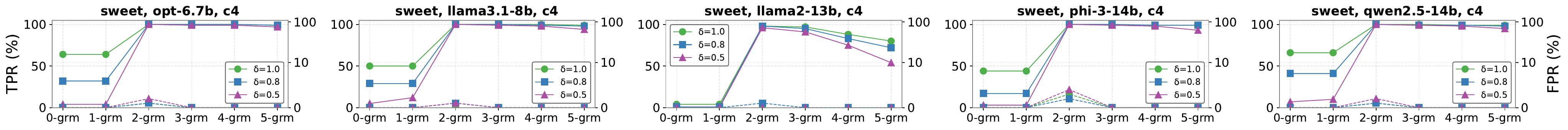}
    \includegraphics[width=\textwidth]{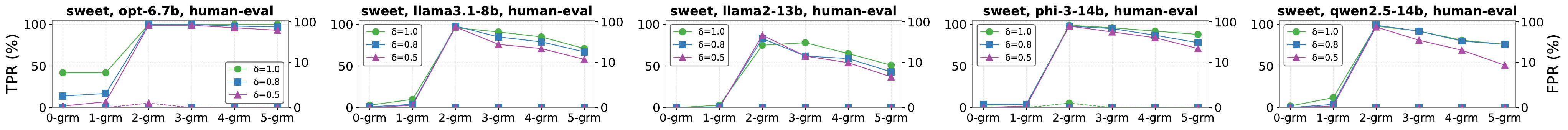}
    \includegraphics[width=\textwidth]{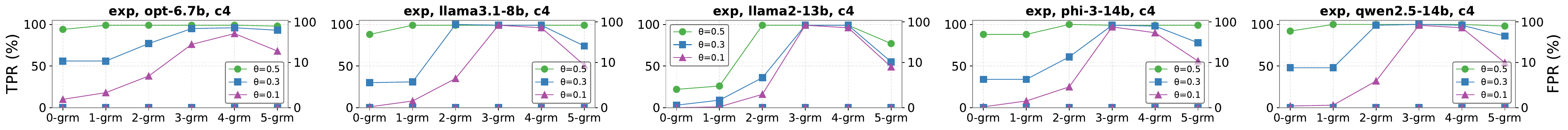}
    \includegraphics[width=\textwidth]{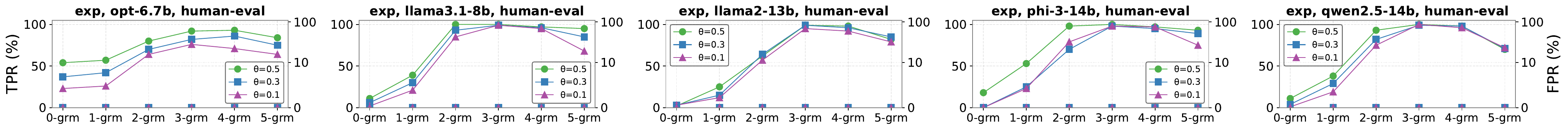}
    \caption{\new{This figure replicates experiments in Fig.\,\ref{fig:data1000} across LLMs. From top to bottom: \textsc{Sweet}+C4, \textsc{Sweet}+\textsc{HumanEval}, \textsc{Exp}+C4, and \textsc{Exp}+\textsc{HumanEval}. The results show that the qualitative trend of filtering is insensitive to the text generator.}}
    \label{fig:data1000_2}
\end{figure*}

\new{Figs.\,\ref{fig:data1000}--\ref{fig:data1000_2} summarize complementary views of signature filtering in weak-signal settings. Fig.\,\ref{fig:data1000} focuses on Llama3.1 on high-entropy (C4) and low-entropy (\textsc{Mbpp} and \textsc{HumanEval}) texts and shows TPR/FPR by $n$-gram order. Fig.\,\ref{fig:data1000_2} replicates the TPR/FPR curves for \textsc{Kgw} and \textsc{Exp} over C4 and HumanEval across five language models.}

As expected, all watermark families attain near-optimal FPR at 0-gram, but their TPR drops quickly as the watermark strength weakens. Introducing an $n$-gram signature consistently lifts TPR, with two caveats. First, 1-gram signatures are too coarse. Since type-level deletion by the signature ignores position-dependent coloring, the solver seldom finds effective 1-grams to tell apart watermarked and natural texts. This failure mode is stark for \textsc{Unigram}, where token color is tied to token type across documents. This leads the 1-gram filter to remove the same red types in both classes, pushing both TPR and FPR to 100\%. Second, different watermarks have different sweet spots. For \textsc{Kgw}-style schemes (\textsc{Kgw}, \textsc{Sweet}, and \textsc{Unigram}), 2-grams strike the best balance across strengths; for \textsc{Exp}, useful signal sits in slightly longer local patterns, making 3-gram or 4-gram signatures more effective.

\new{Fig.\,\ref{fig:data1000_2} shows that these sweet spots are remarkably stable across text generators:
2--3-gram signatures consistently deliver the largest TPR gains for \textsc{Kgw} at low FPR, while 3--4-grams do the same for \textsc{Exp}. The spread in TPR/FPR across models at a fixed $n$-gram order is modest compared with the jump from 0-gram to the best $n$-gram filter, indicating that signature design is largely governed by the watermark family and the underlying corpus statistics rather than idiosyncrasies of a particular LLM.
Generally, high-order patterns become more scarce as the watermark strength declines. Hence, long $n$-grams either occur too infrequently to contribute useful deletions or over-delete idiosyncratic fragments, reducing the effective sample and dampening detection. This explains why, even though filtering still beats the baseline as the signal weakens, high $n$-gram orders no longer improve on the sweet spot.} 


\subsection{Detection Capability under Text Edits}
\label{sec:exp-robustness}

\begin{figure}[h]
    \centering
    \includegraphics[width=\columnwidth]{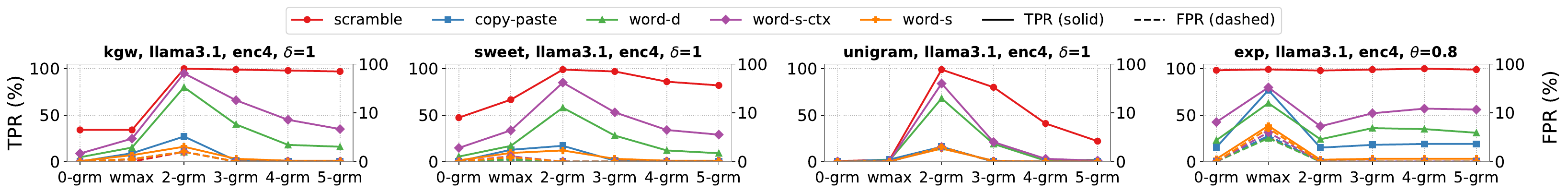}
    \includegraphics[width=\columnwidth]{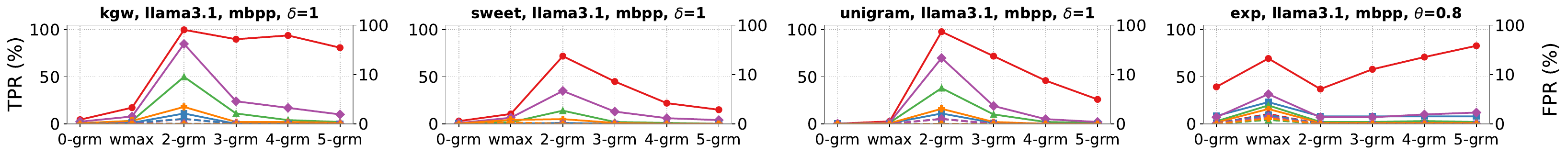}
    \caption{
        \new{The TPR/FPR of $n$-gram signature filtering for Llama3.1-8b on C4 (top) and \textsc{Mbpp} (bottom) in the low-strength ($\delta=1$ and $\theta=0.8$) and edited-text regimes. 0-gram means no filter, and \texttt{wmax} corresponds to the \textsc{WinMax}-enhanced baseline.}
    }
    \label{fig:robustness}
\end{figure}

We stress-test the signatures from Sec.\,\ref{sec:exp-accuracy} against five types of text edits (see \citealt{kirchenbauer2024reliabilitywatermarkslargelanguage,pan2024markllmopensourcetoolkitllm}): sentence-level shuffling (\emph{scramble}) and dilution (\emph{copy-paste} 25\% watermarked text into a natural host), plus three word-level perturbations---30\% deletion (\emph{word-d}), 50\% synonym replacement via WordNet (\emph{word-s}), and 50\% context-aware substitution via BERT (\emph{word-s-ctx}).
We compare our approach with \textsc{WinMax} \cite{kirchenbauer2024reliabilitywatermarkslargelanguage}, an advanced detection-time module that enhances watermark signals to resist text edits. Instead of testing the whole document once, \textsc{WinMax} slides overlapping windows of various lengths across the text, recomputing the watermark score on each window. The maximum per-window score is used for the final decision. 

Fig.\,\ref{fig:robustness} outlines comparison results over both high-entropy (C4) and low-entropy (\textsc{Mbpp}) corpora. Overall, \textsc{Kgw}-style watermarks exhibit short-range dependence: the 2-gram order matches the correlation length and remains effective after edits, obtaining consistent gains over \textsc{WinMax}. Increasing the signature order beyond 2-grams generally reduces detection power: word-level edits quickly destroy longer contexts, making higher-order patterns more difficult to match and thereby less effective. \neww{The contrast with \textsc{Exp} is consistent with Sec.\,\ref{sec:exp-accuracy}, where \textsc{Exp} benefits most from higher-order signatures even on the original texts. Word-level edits break the longer contexts these signatures rely on, making filtering ineffective over all $n$-gram orders.}

\neww{To quantify when the adaptive-window scoring of \textsc{WinMax} has an advantage, we compute a post-hoc localization metric $\Delta_{\mathrm{win}} \coloneqq \max_W Z(W) - Z(T)$ on the edited text $T$, where $\max_W Z(W)$ is the best window z-score identified by \textsc{WinMax}, and $Z(T)$ is the z-score of the full document. Larger $\Delta_{\mathrm{win}}$ therefore indicates that the surviving watermark evidence is more spatially concentrated, so window-based detection has more room to improve over a single global filter.
This localization effect becomes more pronounced after edits: in the same stress test, the mean $\Delta_{\mathrm{win}}$ of \textsc{Exp} exceeds that of \textsc{Kgw}, \textsc{Sweet}, and \textsc{Unigram} by about 0.5--1.2 across edit types. This indicates that surviving evidence is more spatially concentrated for \textsc{Exp} than for \textsc{Kgw}-style watermarks. A fixed signature may therefore lose coverage of part of the surviving signal, whereas \textsc{WinMax} can re-localize it by scanning all possible windows on the edited text.}

Accordingly, in Fig.\,\ref{fig:robustness}, \textsc{WinMax} consistently attains higher TPR on \textsc{Exp}, though at the cost of a slightly higher FPR due to the maximizing effect. By contrast, signatures improve the global green ratio in \textsc{Kgw}-style watermarks by excising many small red clusters whose locations remain relatively stable even after editing. Across both watermark families and corpora, filtering enhances detection with little FPR inflation, consistent with our false-positive analyses in Sec.\,\ref{sec:FPR-analysis}.

\subsection{Scalability of Signature Filters}
\label{sec:scalibility}

\begin{table*}[ht]
\centering
\caption{\new{Watermark performance on the first $\mathrm{50k}$ texts in \textsc{Code-Search-Net} (left block) and C4 (right block) under a 60s soft time limit per signature. This table fixes $B\!=\!\mathrm{1k}$ and varies $G$, shifting from the most aggressive bagging ($G\!=\!\mathrm{1k}$) to no bagging ($G\!=\!\mathrm{50k}$). The columns show the TPR, FPR, and F1 Score of the watermark schemes in percent.
The last row provides the average computation time for each signature. The average filtering time is within milliseconds per signature.}}
\label{tab:scalability}
\resizebox{\textwidth}{!}{%
\begin{tabular}{|l|ccc|ccc|ccc|ccc|ccc|ccc|ccc|ccc|} 
\toprule

& \multicolumn{3}{c}{\textbf{Kgw}} & \multicolumn{3}{c}{\textbf{Sweet}} & \multicolumn{3}{c}{\textbf{Unigram}} & \multicolumn{3}{c|}{\textbf{Exp}} 
& \multicolumn{3}{c}{\textbf{Kgw}} & \multicolumn{3}{c}{\textbf{Sweet}} & \multicolumn{3}{c}{\textbf{Unigram}} & \multicolumn{3}{c|}{\textbf{Exp}} \\
\cmidrule{2-4} \cmidrule{5-7} \cmidrule{8-10} \cmidrule{11-13}
\cmidrule{14-16} \cmidrule{17-19} \cmidrule{20-22} \cmidrule{23-25}

& \scriptsize TPR & \scriptsize FPR & \scriptsize F1 
& \scriptsize TPR & \scriptsize FPR & \scriptsize F1 
& \scriptsize TPR & \scriptsize FPR & \scriptsize F1 
& \scriptsize TPR & \scriptsize FPR & \scriptsize F1 
& \scriptsize TPR & \scriptsize FPR & \scriptsize F1 
& \scriptsize TPR & \scriptsize FPR & \scriptsize F1 
& \scriptsize TPR & \scriptsize FPR & \scriptsize F1 
& \scriptsize TPR & \scriptsize FPR & \scriptsize F1 \\
\midrule
$G\!=\!\mathrm{1k}$   & 99 & 4 & 98 & 99 & 1 & 99 & 99 & 14 & 93 & 78 & 0 & 87    & 99 & 17 & 92 & 99 & 3 & 98 & 99 & 3 & 98 & 95 & 0 & 97 \\
$G\!=\!\mathrm{10k}$  & 94 & 2 & 96 & 82 & 0 & 90 & 73 & 11 & 79 & 43 & 0 & 60    & 91 & 10 & 91 & 72 & 1 & 83 & 65 & 2 & 78 & 72 & 0 & 84 \\
$G\!=\!\mathrm{25k}$  & 92 & 1 & 95 & 76 & 0 & 86 & 67 & 10 & 76 & 37 & 0 & 54    & 87 & 8 & 89 & 63 & 1 & 77 & 58 & 1 & 73 & 67 & 0 & 80 \\
$G\!=\!\mathrm{50k}$  & 89 & 1 & 94 & 69 & 0 & 82 & 62 & 9 & 73 & 32 & 0 & 48    & 83 & 6 & 88 & 57 & 0 & 73 & 54 & 1 & 70 & 62 & 0 & 77 \\
No filter& 10 & 0 & 18 & 10 & 0 & 18 & 18 & 5 & 29 & 31 & 0 & 47    & 9 & 0 & 17 & 8 & 0 & 15 & 11 & 0 & 20 & 61 & 0 & 76 \\
\hline
Sol.~time & \multicolumn{3}{c|}{5 sec.} & \multicolumn{3}{c|}{4 sec.} & \multicolumn{3}{c|}{10 sec.} & \multicolumn{3}{c|}{67 sec.} & \multicolumn{3}{c|}{4 sec.} & \multicolumn{3}{c|}{4 sec.} & \multicolumn{3}{c|}{6 sec.} & \multicolumn{3}{c|}{61 sec.} \\
\bottomrule
\end{tabular}}
\end{table*}


\begin{table*}[ht]
\centering
\caption{\new{Ablation on $B$ with $G\!=\!1.2$k, evaluated on the first 1.2k texts in \textsc{Code-Search-Net} (left block) and C4 (right block).}}
\label{tab:ablation_B_g1200}
\resizebox{\textwidth}{!}{%
\begin{tabular}{|l|ccc|ccc|ccc|ccc|ccc|ccc|ccc|ccc|} 
\toprule

& \multicolumn{3}{c}{\textbf{Kgw}} & \multicolumn{3}{c}{\textbf{Sweet}} & \multicolumn{3}{c}{\textbf{Unigram}} & \multicolumn{3}{c|}{\textbf{Exp}} 
& \multicolumn{3}{c}{\textbf{Kgw}} & \multicolumn{3}{c}{\textbf{Sweet}} & \multicolumn{3}{c}{\textbf{Unigram}} & \multicolumn{3}{c|}{\textbf{Exp}} \\
\cmidrule{2-4} \cmidrule{5-7} \cmidrule{8-10} \cmidrule{11-13}
\cmidrule{14-16} \cmidrule{17-19} \cmidrule{20-22} \cmidrule{23-25}

& \scriptsize TPR & \scriptsize FPR & \scriptsize F1 
& \scriptsize TPR & \scriptsize FPR & \scriptsize F1 
& \scriptsize TPR & \scriptsize FPR & \scriptsize F1 
& \scriptsize TPR & \scriptsize FPR & \scriptsize F1 
& \scriptsize TPR & \scriptsize FPR & \scriptsize F1 
& \scriptsize TPR & \scriptsize FPR & \scriptsize F1 
& \scriptsize TPR & \scriptsize FPR & \scriptsize F1 
& \scriptsize TPR & \scriptsize FPR & \scriptsize F1 \\
\midrule
$B\!=\!1.2$k  & 100 & 0 & 100 & 100 & 0 & 100 & 100 & 9 & 96 & 80 & 0 & 89 & 100 & 8 & 96 & 99 & 4 & 98 & 100 & 1 & 100 & 94 & 0 & 97 \\
$B\!=\!0.6$k  & 90  & 0 & 95  & 89  & 0 & 94  & 78  & 8 & 84 & 55 & 0 & 71 & 88  & 4 & 92 & 82 & 2 & 89 & 73  & 0 & 84  & 79 & 0 & 88 \\
$B\!=\!0.4$k  & 84  & 0 & 91  & 81  & 0 & 90  & 66  & 8 & 76 & 46 & 0 & 63 & 80  & 2 & 88 & 73 & 1 & 84 & 61  & 0 & 76  & 74 & 0 & 85 \\
$B\!=\!0.3$k  & 78  & 0 & 88  & 76  & 0 & 86  & 60  & 7 & 72 & 43 & 0 & 60 & 75  & 2 & 85 & 68 & 0 & 81 & 54  & 0 & 70  & 72 & 0 & 84 \\
No filter & 11 & 0 & 20 & 9 & 0 & 16 & 18 & 6 & 29 & 32 & 0 & 48 & 8 & 0 & 15 & 8 & 0 & 15 & 11 & 0 & 20 & 64 & 0 & 78 \\
\hline
Sol.~time & \multicolumn{3}{c|}{3 sec.} & \multicolumn{3}{c|}{4 sec.} & \multicolumn{3}{c|}{8 sec.} & \multicolumn{3}{c|}{62 sec.} & \multicolumn{3}{c|}{4 sec.} & \multicolumn{3}{c|}{4 sec.} & \multicolumn{3}{c|}{5 sec.} & \multicolumn{3}{c|}{65 sec.} \\
\bottomrule
\end{tabular}}
\end{table*}

\OMIT{To evaluate the effectiveness of signatures over a large corpus, we consider two operational modes:
(i) \emph{Predictive filtering}. In the predictive mode, a single signature learned from the first 1k texts is reused to filter all subsequent texts.
(ii) \emph{Bag-$G$} divides the corpus into groups of $G$ texts. For each group, it computes a signature for the \emph{first} 1k texts. At detection time, the filter applies these signatures separately to the entire corpus and adopts the \emph{largest} z-score for the final decision.
Fig.\,\ref{fig:data5000} outlines the results based on the first 5k texts in C4.

Predictive filtering causes the least TPR and FPR inflation; both metrics steadily \emph{decrease} with corpus size.
The downward trend arises because the fixed signature primarily targets distribution-level disruptions (e.g., common red-leaning $n$-grams) that remain stable in held-out data; as the evaluation set grows, training-only patterns contribute less, yielding less overfitting (lower FPR) and slightly tempered gains (lower TPR).
The 2-gram signatures still significantly improve the no-filter \textsc{Kgw}-style baseline, but higher-order ($n\!\ge\! 3$) $n$-grams degrade rapidly: since longer patterns are rarer in distribution, the patterns learned from the training data tend to overfit and under-generalize when applied to unseen texts. For \textsc{Exp}, the signal concentrates in higher-order local correlations; the 2-gram predictive filter hence often misses the relevant structure and yields only marginal change.

Bagging consistently achieves near-optimal detection with high-order $n$-gram signatures across watermarks.
The bagging filter's high TPR is expected since the TPR of the or-ensemble of $\{f_i\}_i$ is $\max_i\mathrm{TPR}(f_i)$. What may be surprising is its empirical FPR, which is much lower than the worst-case union bound $\sum_i \mathrm{FPR}(f_i)$. A closer examination reveals that misclassified natural texts are mainly those containing numerous \textit{low-entropy} red tokens. These texts are scattered evenly in the corpus and mislead filters trained on different groups. This fact explains why, for example, one 2-gram signature already has an FPR as high as 5\% for \textsc{Kgw}, but five 2-gram signatures only jointly raise the FPR to 7\% (Fig.\,\ref{fig:data5000}, bottom left). It also explains why \textsc{Sweet}, which ignores low-entropy tokens in the text on scrutiny, consistently has the lowest post-filter FPR among these \textsc{Kgw}-style watermarks, whereas \textsc{Unigram}, which shares the same red token types across watermarked and natural texts, consistently has the highest post-filter FPR.


The empirical evidence suggests that bagging is the preferred operational mode to apply a signature filter, as successive local training data are sufficient to sustain global detection. In practice, signatures can be learned on rolling groups of freshly generated text, allowing the detector to adapt on the fly to data drift by aggregating individual scores. 
}


\new{To evaluate the performance of signature filtering over large corpora, we consider a parameterized deployment strategy
called \emph{Bag($B$,$G$)}. At generation time, the training data is divided into groups of $G\ge1$ texts. For each group, we compute a signature on the first $B\le G$ texts in that group. At detection time, all group signatures are applied to the input document, and the largest z-score is used for the decision.}

Tables\,\ref{tab:scalability}--\ref{tab:ablation_B_g1200} report end-to-end performance of signature filtering over 50k low-entropy \textsc{Code-Search-Net} Python snippets and 50k high-entropy \textsc{C4} documents on Opt-1.3b. We focus on the low strength regime $\delta=0.5$ and $\theta=0.3$, using signatures of 2-grams for \textsc{Kgw} and \textsc{Sweet}, 3-grams for \textsc{Unigram}, and 4-grams for \textsc{Exp}. Table~\ref{tab:scalability} fixes $B\!=\!\mathrm{1k}$ and considers $G \in \{$1k, 10k, 25k, 50k$\}$ over the entire datasets, while \new{Table~\ref{tab:ablation_B_g1200} fixes $G\!=\!\mathrm{1.2k}$ and considers $B \in \{$1.2k, 0.6k, 0.4k, 0.3k$\}$ over the first 1.2k texts.} Several patterns emerge:

\begin{itemize}[leftmargin=1.2em, itemsep=2pt, topsep=5pt]
\item \textit{\textsc{Kgw}/\textsc{Sweet} maintains a controllable TPR--FPR frontier.}
On \textsc{Code-Search-Net}, Bagging attains near-optimal TPR (99\%) with low FPR (1--4\%) for $G\!=\!\mathrm{1k}$. Increasing \textit{G} trades a small TPR loss for further FPR reductions.
On C4, the same frontier shifts upward in FPR for $G\!=\!\mathrm{1k}$, while $G\!=\!\mathrm{25k}$ tightens FPR to single digits at modest TPR cost. 
These patterns reflect the effect of text entropy: high-entropy corpora demand larger $G$ to suppress spurious red clusters.


\item \textit{\textsc{Unigram} shows local-context reins in FPR.}
Using 3-grams markedly lowers FPR relative to shorter context:
on C4, FPR lands in the 1–3\% range across modes with high TPR; on Code-Search-Net, FPR is higher but now bounded near 10–14\%.
The gain comes from breaking type-level coupling with light contextualization, which curbs the spurious flips seen with lower-order filters.

\item \textit{\textsc{Exp} achieves near-zero FPR with strength-dependent TPR.}
In Table\,\ref{tab:scalability}, \textsc{Exp} consistently exhibits near-zero false positives, while its true positive rates vary in the expected way with mode coverage. Because \textsc{Exp} evidence tends to appear in short bursts, larger group sizes can miss some group-specific local structure and thus lose recall. For $G\!=\!\mathrm{50k}$, filtering shows nearly no benefit over the baseline.
\end{itemize}


\neww{It is worth noting that the FPR trends in Tables \ref{tab:scalability}--\ref{tab:ablation_B_g1200} differ systematically with corpus entropy across watermark families. Specifically, \textsc{Kgw} exhibits lower false positives on the low-entropy than high-entropy corpora, while \textsc{Unigram} shows the reverse pattern. A plausible explanation is that, in \textsc{Kgw}, the green/red partition varies across positions, so recurring short templates in low-entropy code can be filtered more consistently. Bagging in small groups on high-entropy text is more prone to fitting idiosyncratic red clusters, and thus tends to increase false alarms. \textsc{Sweet} resembles \textsc{Kgw} but ignores low-entropy tokens in the z-test. This attenuates the entropy effect and helps keep \textsc{Sweet}'s false positives low on both corpora. By contrast, \textsc{Unigram} ties token color to token type across documents. This coupling is stronger in low-entropy code, where the same token types recur across code snippets heavily, and it is weaker in high-entropy natural language texts with greater lexical diversity. For \textsc{Exp}, the false positives remain essentially near zero across both corpora, which is consistent with Theorem~\ref{thm:exp-universal} under the score secrecy assumption.}

\new{
These results provide practical guidance to tune the parameters $(B,G)$ along the TPR--FPR frontier.
Generally, reducing $G$ (or increasing $B$) raises both TPR and FPR by producing more idiosyncratic signatures whose ensemble maximum tends to exceed the threshold. In contrast, increasing $G$ (or reducing $B$) accelerates the decline of both TPR and FPR, as a less representative signature generalizes more weakly.
When $B$ is too small relative to $G$, signatures fail to capture meaningful corpus-level patterns, and filtering degrades toward the baseline behavior. 
The effects of batch and group size follow from standard underfitting vs.~overfitting considerations and the multiple-testing nature of ensemble learning \cite{polikar2012ensemble}. 
}

\OMIT{\subsection{Discussion}

Across four watermark schemes on high‑entropy and low‑entropy corpora, signature filtering boosts true‑positive rates in the weak‑signal and low‑variation regimes where baseline detectors struggle. 
Our empirical results show that watermark families differ significantly in their interaction with filtering. As a rule of thumb, we suggest the following guidance for deployment: (i) adopt 2‑gram signatures for \textsc{Kgw} and \textsc{Sweet}, (ii) use 3‑gram for \textsc{Unigram} and \textsc{Exp}, (iii) set Bag-1k as the default streaming mode, and (iv) implement post‑deployment calibration of parameters using in‑domain holdout data with periodic data drift checks.
}

\section{Discussion and implications}
\label{sec:discussion}


\subsection{Theoretical implications for LLM watermarking}

\new{Our analysis shows that filtering at detection time can be added on top of existing watermark tests. It does not require statistical assumptions that are stronger than those already used by the baseline detector. Under the standard \textsc{Kgw} coloring model, deleting a fixed subset of tokens from an unwatermarked text does not change the null distribution of the z-statistic. It mainly reduces the effective sample size, which slightly widens the confidence intervals. This independence model is idealized, but it is the same assumption made by the original watermark schemes. \neww{When the deletion rule becomes correlated with hidden colors, the guarantee may weaken, which is precisely the motivation for the threat analysis in Sec.\,\ref{sec:FPR-analysis}. In our experiments on repeated, correlated, and edited texts, signature filtering remains effective even when independence holds only approximately.}}

\new{The false-positive bounds in Sec.\,\ref{sec:FPR-analysis} provide a practical rule of thumb for operating signature filters. In the color-blind setting, one can safely delete a number of tokens that grows linearly with the text length. With this choice, the conditional false positive rate $\Perror = \mathbb P(Z' \ge z_0 \mid Z < z_0, \mathcal H_0)$ can be kept below any chosen tolerance level. The situation changes when an attacker can fully control token colors. In that case, any watermark test that relies only on color statistics is inherently robust only up to $\sqrt{n}$-scale deletions. Below this scale, worst-case false positives remain bounded. Beyond it, even an optimally designed hypothesis test can be forced to flip some correct decisions. This $\sqrt{n}$ threshold reflects a fundamental limit of color-based watermarking, not a specific weakness of signature filtering.}

\new{We also observe consistent locality patterns in Sec.\,\ref{sec:exp-accuracy}. In particular, 2--3-gram signatures work best for \textsc{Kgw}-style schemes, while 3--4-grams work best for \textsc{Exp}. These results suggest that most usable watermark evidence comes from short-range token dependencies. In this sense, signature filtering serves as a lightweight pre-processing layer that concentrates this evidence, along with an explicit and tunable deletion budget that practitioners can manage.}

\subsection{Practical implications for information processing and management}


\medskip\noindent\textbf{Drop-in enhancement and design guidelines.}
Signature filtering operates entirely at detection time. It only requires access to historical
watermarked and natural texts to learn signatures via MILP. This makes it a lightweight plug-in for information retrieval, indexing, and content management pipelines that already deploy watermarks but must operate under strict false-positive budgets.
In large or streaming corpora, we consider a bagging deployment mode, which learns signature batches on rolling text groups and takes the maximum post-filter z-score across these signatures. 
Practitioners can select batch sizes, group sizes, and decision thresholds according to their risk tolerance and computational budget, much as they currently tune watermark strength or operating points of existing detectors.

\medskip\noindent\textbf{How signature filtering differs from other techniques.}
Relative to other detection-time enhancement methods like \textsc{Ewd} and \textsc{WinMax},
our work is more friendly to governance.
Techniques that amplify watermark signals by changing the test statistic introduce
new scores that must be calibrated and explained alongside watermark tests. By contrast,
signature filtering keeps the underlying detection unchanged and instead learns a reusable filter that
modifies which tokens enter the test. This yields (i) explicit control of the additional Type-I error through
$\Perror$, (ii) interpretability in terms of which contexts are removed, and (iii) ease of integration
into policy-driven pipelines where each component must expose its contribution to overall
risk. These features make filtering particularly attractive for management
settings in which detection outputs support triage and provenance attribution.

\medskip\noindent\textbf{\neww{Cross-lingual and multimodal extensions.}}
\neww{
Conceptually, signature filtering can be extended to any watermark test that aggregates per-unit evidence and remains calibrated after deleting units from observable content. For multilingual watermarked documents \citep{he2024watermarkssurvivetranslationcrosslingual}, our two-stage detector may be applied unchanged by modifying the representation used for signature learning.
When multilingual LLMs share the same tokenizer, signatures can be learned either separately per language or jointly with language tags. 
When translation is a routine step, one can instead translate both the training corpus and candidate texts into a pivot language, after which the same signature learning and detecting pipeline can be applied.
Tokens or short phrases can also be mapped into cross-lingual semantic clusters, and the same optimization can be carried out over these high-level units. Signature filtering thus aims to remove ambiguous clusters that dilute evidence \cite{Liu2024bSemanticInvariant}.

For images and other modalities, one can discretize the signal into a finite token stream (e.g., pixels, patches, or quantized latents) and apply pseudorandom coloring or scoring to these units, such that signatures suppress 
regions that consistently provide little evidence or exhibit high variance \citep{Fernandez2023StableSignature}.
Because the deletion rule still depends only on observable features, the null-calibration argument in Sec.\:\ref{sec:prelim} remains valid in these scenarios. Nevertheless, to keep optimization feasible at scale, practical solutions will hinge on precise encoding of sufficient conditions, efficient pruning of candidate patterns, and scalable implementation of operation modes.}

\medskip\noindent\textbf{Robustness to dynamic content distributions.}
In production, a deployed signature may become stale when the deployment context distribution has drifted enough that the post-filter detector no longer satisfies its target TPR/FPR operating point. This is a common issue for machine-learning systems deployed under concept drift or dataset shift \citep{Gama2014Survey,Lu2019ConceptDriftReview}. Gradual covariate drift is expected to reduce recall first, because learned $n$-gram contexts will match less often or remove less red evidence. This behavior is consistent with the smooth TPR decay observed in the Bag$(B,G)$ experiments of Sec.\,\ref{sec:scalibility}. The FPR remains governed by null calibration, but drift toward highly repetitive prose may push signatures for \textsc{Kgw}-style watermarks toward the distributionally correlated regime characterized in Theorem \ref{thm:square-law}.

Operationally, data drift can be monitored using standard methods such as two-sample tests or domain-discriminator tests \citep{Gretton2012KernelTwoSample,Rabanser2019FailingLoudly}. If data drift is found to be fast, the bagging strategy in Sec.\,\ref{sec:scalibility} can be extended with an adaptive moving-window mechanism: signatures are learned from recent batches, stale signatures outside the window are discarded or downweighted, and newly optimized signatures are added as the stream evolves. This mirrors existing adaptive and streaming machine-learning techniques for concept drift \citep{Bifet2007ADWIN}. We note that, while standard monitoring and rolling calibration can enhance a filter's resilience to gradual dynamic content changes, signatures still need to be re-optimized after significant concept shifts, tokenizer or model changes, and watermark key rotation.

\subsection{Limitations and future directions}

Several limitations qualify our contributions and motivate future work. First, our theoretical guarantees assume that the watermark's statistical model remains informative after filtering. This assumption can break if an attacker can manipulate or infer token colors. For example, the attacker might partially recover the green/red partition or obtain model internals. In that case, any color-based detector (including signature filtering) becomes vulnerable, and the linear-safe regime may no longer apply.

A second limitation arises from distributional correlation. When signatures are trained and deployed on highly repetitive corpora, the learned filter can become correlated with benign text. This correlation can mimic adversarial deletions and push the system toward the $\sqrt{n}$-scale fragile regime characterized in Theorem~\ref{thm:square-law}.
In practice, such risk can be mitigated by combining filtering with entropy-aware or semantics-aware mechanisms that down-weight highly predictable tokens. Our experiments already show benefits from such mechanisms (for example, \textsc{Sweet}). An important next step is to fully characterize and evaluate these hybrid systems, determining how to allocate deletion budgets and weights across the different signals.

\neww{Another challenge is learning effective signatures in applications with scarce data, such as highly specialized domains where it is difficult to assemble representative watermarked and natural documents for training. Table \ref{tab:ablation_B_g1200} quantifies this sensitivity, showing that reducing the per-group training size $B$ lowers both TPR and FPR across all watermark families. The gains over the baseline persist in most settings but can become modest. This observation suggests several mitigation strategies, like pooling adjacent batches or reducing the group size $G$ such that each signature is estimated from more representative data, updating signatures incrementally as new texts arrive, and prioritizing conservative signature orders in scarce-data domains to avoid brittle patterns. A systematic investigation of these strategies is our subsequent goal.}

\neww{
A practical deployment limitation is key dependence. The signatures computed by our MILP are based on key-conditioned evidence, so rotating the watermark key changes the optimization objective and weakens the learned signatures. Although the offline recomputation cost is modest for \textsc{Kgw}-style watermarks, it is noticeably higher for \textsc{Exp} (see Tables \ref{tab:scalability}--\ref{tab:ablation_B_g1200}). Accordingly, signature filtering is most attractive when keys are stable over a deployment interval or when re-optimization can run asynchronously in the background.} 

Finally, while our results suggest that signature filtering can serve as a useful component in provenance
pipelines, it does not replace human judgment or eliminate the need for complementary signals. In high-stakes
settings, watermark-based decisions should be interpreted as one input among many, alongside human review,
retrieval-based corroboration, and content-authenticity metadata. Designing such socio-technical systems and empirically studying how managers and end users understand, trust, and act on watermark evidence remains a critical area for future research. 


\OMIT{
\paragraph{Future work}
Signature filtering can be combined with evidence aggregation~\citep{xiong2025delphiagent} and with content‑authenticity signals studied for AI‑generated text and media, offering managers a layered strategy for provenance risk.
Two directions are especially relevant toward this goal. The first direction is \emph{multi‑signal fusion}, which involves combining signature filtering with existing detectors such as detectors based on entropy or attention patterns \cite{wang2025benfordipm}. The resulting hybrid system is then embedded into a content-trust pipeline to reduce residual risk at fixed review budgets~\citep{xylogiannopoulos2024chatgpt}. The second direction is \emph{policy integration}, which translates our conditional false-positive guarantees into service-level objectives, e.g., ``false-positive rate $\le$ 0.10\% at one-million items per day'' or ``95\% of borderline documents reviewed within 2 hours.''
These objectives can then be applied to indexing, moderation, and academic-integrity workflows, enabling organizations to translate detection performance into auditable thresholds and embed them in formal disclosure policies.


\OMIT{
\medskip
\noindent
\textbf{Implications for information systems.}
From the perspective of content provenance and information integrity, signature filtering functions as a \emph{drop-in module} for existing watermark pipelines:
(i) it leaves generation untouched, 
(ii) it is compatible with multiple detectors, and 
(iii) it can be computed once (or amortized over rolling groups) and reused at scale.
This makes it suitable for search and repository operators who must decide—often automatically and at the \emph{segment level}—whether text likely contains AI-generated content, with downstream consequences for indexing, ranking, and policy enforcement.

\medskip
\noindent
\textbf{Managerial takeaways.}
\begin{itemize}
  \item \emph{When to activate the filter.} Use signature filtering as a second-stage check \emph{only} when the baseline test is inconclusive; this design limits the failure mode to flipping a correct ``unwatermarked'' decision and lets risk be quantified as a \emph{conditional} Type-I error.
  \item \emph{Default choices.} For K\textsc{gw}-style and related schemes, $2$-gram signatures offer the best accuracy–cost balance in our experiments; for Unigram and Exp variants, prefer conservative deletion budgets and monitor solver timeouts.
  \item \emph{Deployment guardrails.} Enforce a residual length (e.g., $\ge 30$ tokens) to keep normal approximations valid; calibrate operating points on held-out natural text; and cap deletions by the safe budgets from \S\ref{sec:FPR-analysis} (linear in $n$ under color-blind assumptions; $\tilde O(\sqrt n)$ fragility under color-adaptive or distributionally correlated scenarios).
  \item \emph{Streaming at scale.} Use \emph{bagging} over small, recent groups to adapt to drift; in practice this retained near-oracle TPR while keeping FPR well below naive union-bound estimates, at modest compute overhead.
\end{itemize}

\medskip
\noindent
\textbf{Limitations and ethics.}
Signature filtering assumes the watermark’s statistical model remains informative after filtering; adversaries who adapt token colors or engineer distributional correlation can erode guarantees.
To reduce unintended FPR inflation on benign low-entropy material (e.g., boilerplate, code), pair filters with entropy-aware prechecks and log flips for audit.
The method should support \emph{triage and provenance attribution}, not punitive adjudication in isolation; human review and multi-signal evidence (e.g., metadata, retrieval-based corroboration) remain essential in high-stakes settings.

\medskip
\noindent
\textbf{Reproducibility and reuse.}
We release code and scripts for signature construction and evaluation, together with instructions to reproduce the figures end-to-end.
For operational reuse, we provide default configurations (recommended $n$-gram order, safe deletion caps, residual-length thresholds) and example pipelines for both predictive and bagging modes.


\medskip
\noindent
In sum, signature filtering offers a practical, theory-backed, and deployment-ready enhancement for watermark detection that aligns with information-processing needs: it improves recall where managers need it most, preserves calibrated risk, and scales to real corpora with simple operational guardrails.

}
\section{Conclusion}

We presented \emph{signature filtering}, a detection-time plug‑in that deletes a compact pre‑learned set of statistically disruptive tokens before running a baseline watermark test. Under the \textsc{Kgw} assumption, the approach preserves the null distribution of the z‑test, and our finite‑sample and asymptotic analyses delineate safe deletion budgets and quantify worst‑case false‑positive risks. 

\paragraph{Implications for IP\&M practice.}
Signature filtering is designed to slot into existing information pipelines (see Fig.\,\ref{fig:workflow}), enabling several management‑facing uses:
\begin{enumerate}[leftmargin=*,itemsep=1pt,topsep=2pt]
  \item \textbf{Provenance at index time.} As a pre‑index gate or as a ranking feature, the filter raises separation between natural and watermarked text in predictable prose, improving \emph{index hygiene} without modifying generation. This supports content provenance, collection curation, and retrieval quality control at scale.
  \item \textbf{Repository curation and audit.} Document‑ or segment‑level flags produced by the filter can be logged as metadata for downstream auditing, compliance, and lifecycle policies (e.g., differential retention or human review), complementing verification frameworks in IP\&M that reason over multi‑agent or multi‑signal evidence chains~\citep{xiong2025delphiagent}.
  \item \textbf{Platform governance and trust.} For user‑generated content such as product reviews, where AI paraphrases can erode consumer trust, provenance signals can inform moderation and transparency policies~\citep{xylogiannopoulos2024reviews}. Our filter provides a model‑agnostic signal that can be fused with other cues (style, network, or attention‑law features)~\citep{wang2025benfordipm}.
  \item \textbf{Streaming operations.} Bagging over rolling groups amortizes signature computation and sustains high TPR with low FPR inflation in streams, aligning with continuous ingestion settings (news feeds, enterprise knowledge bases).
\end{enumerate}

\paragraph{Connections to managerial research.}
Our results contribute a \emph{provable, deployable} component to broader IP\&M agendas on information quality, credibility, and automated verification. Signature filtering can be combined with evidence aggregation (e.g., agent‑based verifiers)~\citep{xiong2025delphiagent} and with content‑authenticity signals studied for AI‑generated text and media~\citep{wang2025benfordipm}, offering managers a layered strategy for provenance risk.

\paragraph{Limitations and guidance.}
Theory and experiments highlight two practical cautions. First, distributional correlation between training and deployment data can inflate conditional FPR if deletion budgets are too aggressive; our bounds and ablations provide \emph{safe operating regions} and emphasize moderate $n$‑gram orders. Second, watermark families differ in their interaction with filtering (e.g., Unigram can raise FPR if operated at 1‑gram with aggressive deletion). As a rule of thumb for production, we recommend: (i) 2‑gram signatures for Kgw and Sweet, (ii) 3‑gram for Exp with stricter deletion caps, (iii) \emph{bagging‑1k} as the default streaming mode, and (iv) post‑deployment calibration of thresholds using in‑domain holdout data with periodic drift checks.

\paragraph{Outlook.}
Two directions are especially relevant for IP\&M. First, \emph{multi‑signal fusion}: combining signature filtering with orthogonal detectors (e.g., entropy‑ or attention‑law‑based) and trust pipelines can reduce residual risk at fixed review budgets~\citep{wang2025benfordipm,xiong2025delphiagent}. Second, \emph{policy integration}: translating conditional‑FPR guarantees into service‑level objectives for indexing, moderation, and academic‑integrity workflows will help organizations set auditable thresholds and disclosure policies. 

\subsection{Concluding Remarks}

\OMIT{
Signature filtering amplifies watermark signals while maintaining control over false positives. Comprehensive experiments across watermark algorithms, attack scenarios, model sizes, and corpus scales confirm it as an effective detection-time upgrade. For \textsc{Kgw} and \textsc{Sweet}, 2-gram signatures offer the best trade-off between detection power and solver cost; by contrast, they tend to raise the FPR for \textsc{Unigram} and trigger more solver timeouts on \textsc{Exp}, indicating the need for scheme-specific tuning. In practice, lightweight bagging sustains near-optimal performance at scale, and future work should explore richer aggregation techniques (e.g., majority voting, hierarchical filters) as well as faster signature construction via continuous relaxations or randomized rounding.
}

Signature filtering amplifies watermark signals while maintaining control over false positives.
Our finite-sample bounds and adversarial analysis delineate safe deletion budgets and quantify the worst-case FPR.
Experiments across watermark families, attack suites, model sizes, and corpus scales further validate the empirical performance. Signatures trained on a small seed corpus can potentially retain high detection power when applied to unseen text; lightweight bagging preserves near-oracle performance even at streaming scale.  
Among the settings we examined, 2-grams offer the best detection-cost balance for \textsc{Kgw} and \textsc{Sweet}, but raise the FPR for \textsc{Unigram} and incur more solver timeouts on \textsc{Exp}, signaling the need for tailored tuning.  
Future work could explore richer aggregation strategies (e.g., majority voting, hierarchical filters) and faster signature construction via relaxations or randomized rounding.
}




\appendix
\section{MILP Formulation for \textsc{Exp} Signatures}
\label{app:aar-obj}

\OMIT{Our MILP formulation requires a linear sufficient condition
for watermark detection. Fix a text $T=\langle t_1,\dots,t_n\rangle$.
For \textsc{Kgw}-style watermarks,
we have used $N_g(T) \ge \rho n$ as a sufficient condition for testing
$N_g(T) \ge \gamma n + z_0 \sqrt{\gamma(1-\gamma)n}$, where $\rho$ is determined
by $\gamma$, $n$, and $z_0$ (see Sec.\,\ref{sec:ILP-formulation}).
For \textsc{Exp} watermarks, we must ensure that
the linear sufficient condition implies $\mathbb{P}(X_n \!\ge\! X_T) < \alpha$,
where $X_n \sim \Gamma(n,1)$, and $X_T \coloneq \sum_{i=1}^n c_i$
is the sum of the per-token \textsc{Exp} scores.}

In the following, let \(X_n\sim\Gamma(n,1)\) with \(n\ge 1\) and let $\gamma(x,n) \coloneqq \mathbb{P}(X_n \ge x)$ denote its survival function.
We first present a sufficient condition that makes \textsc{Exp} watermarks compatible with our signature constraints.

\begin{proposition}
Fix a significance level $\alpha \in (0,1)$ and define \(c(\alpha)>1\) as the unique root of  
$c\,e^{-(c-1)}=\alpha$ on $(1,\infty)$.
Then $x \ge c(\alpha)\,n$ is a sufficient condition for $\gamma(x,n)<\alpha$ and $n>1$.
\end{proposition}

\begin{proof}
Write \(x=sn\) with a stretch factor \(s>1\).  
Since \(\Gamma(n,1)\) can be expressed as a sum of \(n\) independent \(\mathrm{Exp}(1)\) variables,
we have \(X_n=\sum_{i=1}^{n}Y_i\) with each \(Y_i\sim \mathrm{Exp}(1)\).
For any \(t\in(0,1)\), Markov's inequality gives
$
\mathbb{P}(X_n \ge x)
 =\mathbb{P} \bigl(e^{tX_n}\ge e^{tx}\bigr)
 \le e^{-t x}\,\E[e^{tX_n}]
 = e^{-t x}\,\bigl(\E[e^{tY_1}]\bigr)^{n}
   = e^{-t x}\,(1-t)^{-n},
$
since the m.g.f.~of \(\mathrm{Exp}(1)\) is \(\E[e^{tY_1}]=(1-t)^{-1}\) for \(t<1\).
Choosing the optimal \(t^{\star} = 1-s^{-1} \in(0,1)\) 
then yields the explicit upper tail bound
$\mathbb{P} (X_n\ge s\,n) \le \bigl(s\,e^{-(s-1)}\bigr)^{n}.$

Note that the function \(f(s) = s\,e^{-(s-1)}\) satisfies
\(f'(s) = e^{-(s-1)}(1-s) < 0\) for all \(s>1\), so \(f\) is strictly decreasing on \((1,\infty)\),
with \(f(1)=1\) and \(\lim_{s\to\infty} f(s)=0\). Therefore, for each \(\alpha\in(0,1)\) there is a unique
solution \(c(\alpha)>1\) to \(f(s)=\alpha\).
Let \(s=c(\alpha)\) be this unique solution; plugging it into the upper tail bound yields
$\mathbb{P}(X_n\ge c(\alpha)n)\le\alpha^{n}<\alpha$
for all $n > 1$.
Hence, \(\gamma(x,n)=\mathbb{P}(X_n \ge x)<\alpha\)
whenever \(x\ge c(\alpha)\,n\).
\end{proof}

Our experiments fix $\alpha=10^{-4}$ following the default setting of MarkLLM \cite{pan2024markllmopensourcetoolkitllm}.
Since $|T'| \in [30, 200]$, we can \emph{numerically} compute the best linear sufficient condition for detecting \textsc{Exp} watermarks.

\begin{proposition}
\label{prop:gamma_linear_1.83}
Fix $\alpha = 10^{-4}$ and suppose that $n \in [30, 200]$.
Then $x \ge 1.826\,n$ implies that $\gamma(x,n)<\alpha$.
\end{proposition}

\begin{proof}
Given $n$ and $\alpha$, let \(x^{\star} = x^{\star}(n,\alpha)\) be the unique solution for
$\gamma(x^{\star}\!,n)=\alpha$. (Such a unique solution exists since
$\gamma(0,n)=1$ and $\gamma(x,n) \to 0$ as $x\to\infty$.)
Define \(r_n\coloneqq x^{\star}/n\).
For every integer \(n\in[30,200]\), we can compute $r_n$ numerically.
The maximal ratio occurs at $r_{30} = \max_{30\le n\le200} r_n \approx 1.82505$.
For any constant \(c > r_{30}\), \(n\in[30,200]\), and $x \ge c\,n$, it holds that
$
x \ge c\,n 
      > r_{30}\, n
      \ge r_n\, n
      = x^{\star}.
$
Since \(\gamma(x,n)\) is strictly decreasing in \(x\), we have
$
\gamma(x,n)<\gamma(x^{\star}\!, n) = \alpha.
$
Consequently, \(x\ge1.826\,n\) ensures that
\(\gamma(x,n) < \alpha\).
\end{proof}

\OMIT{
\section{Hardness of Optimal Signature Selection}
\label{app:sig-complexity}

We define the \emph{optimal signature selection} problem as follows. Given $m$ texts $T_1, \dots, T_m$ over vocabulary $\mathcal{V}$ and $m$ thresholds $\rho_1,\ldots,\rho_m$, we want to find a signature $S \subseteq \mathcal{V}$ maximizing the number of texts $T_i$ such that the green fraction of $T_i$ is at least $\rho_i$ after filtering. 

\begin{proposition}
\label{thm:np-hard}
    Optimal signature selection is NP-hard.
\end{proposition}

\begin{proof}
We reduce from the \emph{Maximum Feasible Subsystem (Max-FS) of 0--1 Linear Inequalities} \cite{amaldi1995complexity}, where we are given a set of $m$ linear inequalities in the form of
$\sum_{j=1}^n a_{i,j}\,x_j \ge 0$, with $x_j \in \{0,1\}$ for $i=1,\dots,m$.
The task is to find an assignment for $x_1,\dots,x_n$ that maximizes the number of satisfied inequalities.

We construct an instance of the signature selection problem as follows:
$\mathcal{V}$ consists of $n$ token types $w_1, \dots, w_n$, one for each variable $x_j$.
For the $i$-th inequality $\sum_{j=1}^n a_{i,j}x_j \ge 0$, we create a text $T_i$
according to $a_{i,j}$'s values: 
    \begin{itemize}
        \item If $a_{i,j} > 0$, let $w_j$ appear $a_{i,j}$ times in green and $0$ times in red in $T_i$.
        \item If $a_{i,j} < 0$, let $w_j$ appear $0$ times in green and $-a_{i,j}$ times in red in $T_i$.
        \item If $a_{i,j} = 0$, do not use $w_j$ in $T_i$ (so it does not affect that text).
    \end{itemize}
Finally, we set the threshold $\rho_i = 0.5$ for all $i$.

Observe that, for each text $T_i$, keeping a token $w_j$ (i.e.,\ not including it in the signature) contributes positively if $a_{i,j}>0$ and negatively if $a_{i,j}<0$. More precisely, if we let $K = \{ j : w_j$ is kept$\}$, then $N_{g,i}' \ge 0.5\,N_i'$ if and only if
$
    \sum_{j \in K} (\text{\#green of $w_j$ in $T_i'$}) \ge \sum_{j \in K} (\text{\#red of $w_j$ in $T_i'$})
$
holds for the filtered text $T_i'$. By construction,
$
    \sum_{j \in K} \bigl(\#\text{green of $w_j$ in $T_i'$} - \#\text{red of $w_j$ in $T_i'$}\bigr)
    = \sum_{j=1}^n a_{i,j}\,x_j,
$
with $x_j = 1$ if $w_j$ is kept and $x_j = 0$ otherwise. Thus, $T_i'$ satisfies ${N_{g,i}'} \ge 0.5\,{N_i'}$ if and only if $\sum_{j=1}^n a_{i,j}\,x_j \ge 0.$
Hence, an assignment that satisfies at least $k$ of the original inequalities corresponds precisely to a signature $S = \{w_j : x_j=0\}$ that yields at least $k$ texts with green fraction $\ge 0.5$.

Since Max-FS is NP-complete, this reduction shows that the corresponding
decision version of optimal signature selection is
NP-hard. Membership in NP is immediate, because for any candidate signature
we can compute, in polynomial time, how many texts meet their thresholds and
check whether this number is at least $k$. Hence, the decision version of
optimal signature selection is NP-complete, and the optimization problem is
NP-hard.
\end{proof}



}

\section{Berry--Esseen Remainders}
\label{app:BE-to-remainders}
This section establishes the Berry--Esseen remainders we will use throughout the remaining appendices. 
Let $X_{1},\dots,X_{n}$ be independent random vectors in $\mathbb{R}^{d}$ with
$\mathbb{E}[X_{i}] = 0$. Define
$S_{n}\coloneqq \sum_{i=1}^{n}X_{i}$
with
$\Sigma_{n}\coloneqq \operatorname{Cov}(S_{n})$.
Write $Z_{d}\sim\mathcal{N}(0,I_{d})$ and
$
  \Phi_{d}(A)\coloneqq \mathbb{P}(Z_{d}\in A)
$
for any Borel set $A\subseteq\mathbb{R}^{d}$. Define the
\emph{standardized third moment}
$
  \mu_{n}\coloneqq \sum_{i=1}^{n}
               \mathbb{E}\bigl\lVert\Sigma_{n}^{-1/2}X_{i}\bigr\rVert_{2}^{3}.
$
We use $\Phi(x)$ to denote the standard normal cumulative distribution function, and use $\Phi_d(x_1,\dots,x_d)$ to denote its standard $d$-variate version.

\begin{proposition}[\citealt{bentkus2005lyapunov}]
\label{thm:BE}
There is a constant $c_{\mathrm{BE}}$ such that for every convex Borel set $A\subseteq\mathbb{R}^{d}$, it holds that
$
  \bigl| \mathbb{P} \bigl(\Sigma_{n}^{-1/2}S_{n}\in A\bigr) - \Phi_{d}(A) \bigr|
  \;\le\; c_{\mathrm{BE}}\cdot d^{1/4}\cdot\mu_{n}.
$
\end{proposition}

Recall that $\sigma_{n}\coloneqq\sqrt{\gamma(1-\gamma)n}$
is the standard deviation of the centered Bernoulli
sum \(N_g-\gamma n=\sum_{i=1}^{n}(G_{i}-\gamma)\),
where each \(G_{i}\sim\operatorname{Ber}(\gamma)\)
indicates whether token $i$ is green under the green ratio $\gamma$.
Proposition\,\ref{thm:BE} instantiated into the following bounds for $d=1$ (univariate) and $d=2$ (bivariate):

\smallskip
\noindent\textbf{Univariate remainder $\delta_{n}$.}
Consider the scalar statistic
$Z=(N_g-\gamma n)/\sigma_{n}$. For any $z\in\mathbb R$, it holds that
\begin{align*}
\label{eq:deltaN-derivation}
\Bigl|\mathbb{P}\bigl(Z\le z\bigr)-\Phi(z)\Bigr| \;\le\; c_{\mathrm{BE}}\,\mu_n 
\;=\;c_{\mathrm{BE}}\,\frac{\sum_{i=1}^{n}\mathbb E\lvert G_{i}-\gamma\rvert^{3}}
     {\sigma_{n}^{3}}
\;=\;
c_{\mathrm{BE}}\,\frac{n\gamma(1-\gamma)\bigl[(1-\gamma)^{2}+\gamma^{2}\bigr]}
     {\sigma_{n}^{3}}
\;\le\;\frac{c_{\mathrm{BE}}}{\sigma_{n}} \;\eqqcolon\; \delta_{n}.
\end{align*}

\smallskip
\noindent\textbf{Bivariate remainder $r_{n,n'}$.}
Consider the z-scores
$(Z'\!,\, Z) = \bigl(
  (N_g'-\gamma n')/\sigma_{n'},\,
     (N_g-\gamma n)/\sigma_{n} \bigr)$
after and before signature filtering. The bivariate Berry--Esseen gives
\begin{align*}
\Bigl|\mathbb{P}(Z' < z,\, Z < z)-\Phi_{2}(A)\Bigr|
\;\le
\sup_{A\text{ convex}}
\Bigl|\mathbb{P}\bigl[(Z'\!,\,Z)\in A\bigr]-\Phi_{2}(A)\Bigr|
\;\le\;
\frac{c_{\mathrm{BE}}\cdot 2^{1/4}\cdot 2\sqrt{2}}{\gamma(1-\gamma)\,\rho^3\,\sigma_n}
\;\eqqcolon\; r_{n,n'}.
\end{align*}


\section{False Positive Analysis for Signature Filtering}
\label{app:FPR-analysis}

This section provides the details for Sec.\,\ref{sec:model-A}. We first derive the exact FPR of filtering via binomial-hypergeometric tails. We then offer a Gaussian approximation (Proposition\,\ref{thm:flip-approx}) of the tail probability under a simple variance condition. Finally, we show that the FPR remains below any prescribed tolerance level by limiting the deletion to a linear fraction of the token count and prove Theorem\,\ref{thm:finite-sig-bound} in Sec.\,\ref{sec:model-A}.

\subsection{Exact false-positive probability}
\label{sec:exact-flip}
Fix a signature $S$ and a candidate text $T$ of length $n$.
Let $n'$ denote the number of tokens that survive filtering.
Under $\mathcal H_0$, we can define $N_g \coloneqq  \sum_{i=1}^n G_i$ as the number of green tokens in $T$ under the \textsc{Kgw} coloring assumption, where each $G_i \sim \mathrm{Ber}(\gamma)$.
For a color-blind attacker, we can define $N_g' \coloneqq  \sum_{i=1}^n Y_i\,G_i$ as the number of retained green tokens, where $Y_i$ is an indicator that token~$t_i$ survives. Thus, conditioned on $N_g=x$, it holds that
$
  N_g' \mid (N_g=x) \sim
  \operatorname{HyperGeo}(x,\,n',\,n).
$
Define
$k_z \coloneqq \gamma n + z\sigma_n$ and
$k'_z \coloneqq \gamma n' + z\sigma_{n'}$,
the empirical green fraction thresholds before and after filtering.
The exact probability that a text fails the watermark test before filtering (i.e., $N_g<k_z$) but passes the test after filtering (i.e., $N_g' \ge k'_z$) is
\begin{align}
  \mathbb{P}_{\mathrm{flip}}(n,n',z)
  \;\coloneqq\;
  \sum_{x=0}^{\lfloor k_z-1\rfloor}
  \binom{n}{x}\,\gamma^{x}(1-\gamma)^{n-x}\,p(x),
\end{align}
where $p(x) \coloneqq \mathbb{P}(N_g'\ge k'_z \mid N_g=x)$. 

\subsection{Large-sample normal approximation}
\label{sec:normal-flip}

We may use the central limit theorem to obtain a simple approximation of $\mathbb{P}_{\mathrm{flip}}(n,n',z)$. More precisely, define $Z \coloneqq (N_g - \gamma n) / \sqrt{\gamma(1-\gamma)n}$ and $Z' \coloneqq (N_g' - \gamma n') / \sqrt{\gamma(1-\gamma)n'}$. Recall that $N_g = \sum_{i=1}^n G_i$ and $N_g' = \sum_{i=1}^n Y_i G_i$. The correlation coefficient of $(N_g,N_g')$ is $\rho = \gamma(1-\gamma)n' / (\sqrt{\gamma(1-\gamma)n}\sqrt{\gamma(1-\gamma)n'}) = \sqrt{n'/n}$.

The pair $(N_g,N_g')$ 
satisfies the joint central limit theorem when $n$ and $n'$ are large, in the sense that the joint distribution of $(Z,Z')$ converges in distribution to a bivariate normal vector with correlation coefficient $\rho = \sqrt{n'/n}$.
We use $\Phi_2(\cdot,\cdot\,;\rho)$ to denote the bivariate normal cumulative distribution with correlation $\rho$.

In the rest of this section, we fix an unwatermarked text $T$ under the \textsc{Kgw} coloring assumption, 
and use the conditional probability $\mathbb{P}(Z' \ge z \mid Z < z)$ to quantify the risk that $T$ is flagged as watermarked by a filter. We will refer to this risk as the \emph{(conditional) flip probability} in the sequel.

\begin{proposition}[Large-sample flip probability]
\label{thm:flip-approx}
Recall the univariate reminder $\delta_n$ and the bivariate reminder $r_{n,n'}$
defined in \ref{app:BE-to-remainders}. When $\Phi(z)>\delta_n$, it holds that
\begin{align}
  \left| \mathbb{P}(Z' \ge z \mid Z<z)
  - \frac{\Phi(z) - \Phi_2(z,z;\rho)}{\Phi(z)}\right|
  \;\;\le\;\;
  \frac{r_{n,n'}+\delta_n}{\Phi(z)-\delta_n}.
\end{align}
\end{proposition}
\begin{proof}
Since $N_g=\sum_{i=1}^{n}\!G_i$ and $N'_g=\sum_{i=1}^{n}\!Y_i\,G_i$,
the pair $(N_g,N'_g)$ is a sum of
$\{0,1\}^2$-valued vectors whose third absolute centered moments are
bounded by 1. Hence, the Berry–Esseen bounds in Proposition\,\ref{thm:BE} yield
$
\bigl|\mathbb{P}(Z' < z,\, Z < z) - \Phi_2(z,z;\rho) \bigr| \le r_{n,n'}
$ and $
\bigl|\mathbb{P}(Z < z) - \Phi(x) \bigr| \le \delta_{n}.
$
Taking the difference between these two bounds leads to
$
\mathbb{P}(Z'\ge z,\,Z < z) = \Phi(z) - \Phi_2(z,z;\rho) + \xi_{n,n'}
$
with $|\xi_{n,n'}| \le r_{n,n'} + \delta_n$.

Define
$c\coloneqq\Phi(z)-\Phi_{2}(z,z;\rho)$
and
$\eta_n \coloneqq \mathbb{P}(Z < z) - \Phi(x)$.
Then
$
  \mathbb{P}(Z'\ge z\mid Z<z)
  =\frac{\mathbb{P}(Z'\ge z,\,Z<z)}{\mathbb{P}(Z<z)}
  =\frac{c\,+\,\xi_{n,n'}}{\Phi(z)\,+\,\eta_{n}}.
$
Since \(|\eta_{n}|\le\delta_{n}<\Phi(z)\) by assumption, we have
\(\Phi(z)+\eta_{n}\ge\Phi(z)-\delta_{n}>0\).
It follows that
\begin{align}
  \frac{c+\xi_{n,n'}}{\Phi(z)+\eta_{n}}
  -\frac{c}{\Phi(z)}
   \;=\;\frac{\xi_{n,n'}\Phi(z)\;-\;c\,\eta_{n}}
          {\Phi(z)\bigl[\Phi(z)+\eta_{n}\bigr]}
  \;\le\;
  \frac{|\xi_{n,n'}|\,\Phi(z)+c\,|\eta_{n}|}
       {\Phi(z)\bigl[\Phi(z)-\delta_{n}\bigr]}
  \;\le\;
  \frac{r_{n,n'}+\delta_{n}}
       {\Phi(z)-\delta_{n}},\label{eq:be-bound-2}
\end{align}
where the last inequality follows from \(|c|\le\Phi(z)\) and the Berry--Esseen bounds \(|\xi_{n,n'}|\le r_{n,n'}\) and \(|\eta_{n}|\le\delta_{n}\).
Plugging $\mathbb{P}(Z'\ge z\mid Z<z)=\frac{c\,+\,\xi_{n,n'}}{\Phi(z)\,+\,\eta_{n}}$ and $c = \Phi(z) - \Phi_{2}(z,z;\rho)$ into \eqref{eq:be-bound-2} yields the desired bound.
\end{proof}

\OMIT{
\begin{proposition}[Large-sample flip probability]
\label{thm:flip-approx}
As $n,n'\to\infty$, $(Z,Z')$ converges in distribution to a centered bivariate normal with
the correlation coefficient $\rho=\sqrt{n'/n}$.
Consequently, the flip probability satisfies
\begin{align}
  \mathbb{P}\bigl(Z'\ge z \mid Z<z\bigr)
  \;=\;
  1 - {\Phi_{2}(z,z;\rho)}/{\Phi(z)}.
  \qedhere
\end{align}
\end{proposition}
\begin{proof}
\phantom{\qedhere}
\hspace{-2.5em} Recall that $N_g =\sum_{i=1}^n G_i$ and $N_g' =\sum_{i=1}^n Y_i\,G_i$, and
\begin{align*}
  \mathrm{Cov}(N_g,N_g')
  & \;=\;
  \mathbb{E}[N_g\,N_g'] - \mathbb{E}[N_g]\,\mathbb{E}[N_g']\\
  & \;=\;
  \gamma\,n' - \gamma^2\,n'
    \;=\;
  \gamma\,(1-\gamma)\,n'.
\end{align*}
The correlation coefficient of $(N_g,N_g')$ is therefore
\begin{align*}
  \rho \;=\;
  \frac{\gamma\,(1-\gamma)\,n'}{\sqrt{\gamma\,(1-\gamma)\,n}\,\sqrt{\gamma\,(1-\gamma)\,n'}}
  \;=\;
  \sqrt{\frac{n'}{\,n\,}}.
\end{align*}

When $n$ and $n'$ are large, the pair $(N_g,N_g')$ (which is a sum of $n$ \emph{dependent} indicator variable pairs) satisfies the joint Central Limit Theorem.
Thus, the joint distribution of $(Z,Z')$ converges in distribution to a bivariate normal vector with correlation $\mathrm{Corr}(Z,Z') = \sqrt{n'/n}$.
For large $n$ and $n'$, 
\begin{align*}
  \mathbb{P}\bigl(Z'<z,\,Z\ge z\bigr)
  & \;=\;
  \Phi(z)-\Phi_{2}(z,z;\rho)+o(1).
\end{align*}
Thus, the flip probability can be approximated by
\begin{align*}
    \mathbb{P}(Z' \ge z \mid Z < z)
    & \;=\;
    \frac{\mathbb{P}(Z' \ge  z,\; Z < z)}{\mathbb{P}(Z < z)}\\
    & \;=\;
    1 - \frac{\Phi_2(z,z; \rho)}{\Phi(z)} + o(1).
    \tag*{\qed}
\end{align*}
\end{proof}
}

\OMIT{
Intuitively, the event $\{Z \!<\! z\}$ says that $N_g$ is somewhat below its mean $\gamma n$ by up to $z$ standard deviations $\sigma_n$. This gives us a ``margin'' of about $z\,\sigma_n$ on the pre-filter text $T$.
Since a signature is independent of $T$ under the null hypothesis, it removes random tokens from $T$. If $n'/n \approx 0$ (meaning $\Phi_2(z,z; \rho) \approx \Phi(z)^2$), then random fluctuations can more easily push $N_g'/n'$ to a higher fraction.
If $n'/n \approx 1$ (meaning $\Phi_2(z,z; \rho) \approx \Phi(z)$), the correlation between $N_g$ and $N_g'$ is high enough to keep $N_g'$ also below its own mean $\gamma n'$ by the same $z\,\sigma_{n'}$ margin.
}

\OMIT{

\section{Validity Regime and Rule of Thumb}
The approximation given by Proposition\,\ref{thm:flip-approx} is accurate once the text lengths are large enough. A convenient criterion is
\begin{equation}
  \label{eq:rule-of-thumb}
  n,\, n' \;\ge\; z^{2} / \gamma\,(1-\gamma).
\end{equation}
This condition ensures that the decision threshold $k_{z}$ is not pressed against the extreme right-hand edge, thereby guaranteeing that the bivariate normal tail closely matches the exact binomial-hypergeometric tail.
Concretely, it leaves at least $[z\,(1-\gamma)/\gamma]$ standard deviations of elbow-room between $k_{z}$ and the ceiling $n$. For the parameter range we use (e.g., $z \ge 3$ and $\gamma\approx 0.5$), that margin gives the Gaussian curve enough space to match the discrete tail mass accurately.
(See Lemma~\ref{lem:threshold-gap} in the appendix for the exact inequality.) 

Overall, one should use the Gaussian approximation in Proposition\,\ref{thm:flip-approx} only when \emph{both} of the following conditions hold:
\begin{enumerate}[label=(\roman*)]
    \item\label{cond:variance}
          \emph{Variance check:} The effective sample sizes meet the variance rule of
          thumb~\eqref{eq:rule-of-thumb}.
    \item\label{cond:correlation}
          \emph{Correlation check:} The signature leaves $\ge20\%$ of the tokens in $T$, i.e., $n'/n\ge 0.2$ or $\rho\ge 0.45$.
\end{enumerate}
Why do we need the second check? Intuitively, the approximation in
Proposition\,\ref{thm:flip-approx} treats the \emph{pair}
$(N_g,N_g')$ as a noisy but \emph{highly correlated} bivariate Gaussian. However,
this correlation $\rho=\sqrt{n'/n}$ collapses if we delete too many tokens. A low~$\rho$ means that the retained
green count $N_g'$ is almost independent of the original count~$N_g$,
so the ``Gaussian shadow'' of $(N_g,N_g')$ is forced to mimic a pair of
almost independent normals, which is a poor match to the discrete
reality. Formally, the error bound in the two-dimensional
Berry--Esseen theorem
is proportional to $1/(\rho\sigma_n)$. Thus, the absolute error of the Gaussian approximation increases roughly like $1/\rho$: the bound already doubles when $\rho=0.45$, and explodes as $\rho\to 0$. A 1-or-2 percentage-point absolute error may become a $\ge 20\%$ relative error when $\rho<0.4$.

\OMIT{
A complementary way to see the danger is to picture the Gaussian window. The z-score test declares ``watermarked'' when the centered count $N_g'-\gamma n'$ exceeds $z$ standard deviations $\sqrt{n'\gamma(1-\gamma)}$.
If $n'$ is tiny, this window may cover only one or two possible integer values of $N_g'$, turning the smooth Gaussian curve into a crude step function. With that few lattice points, the Gaussian can misplace much of the true tail probability.
}

\begin{example}
    \label{eg:retention-rate}
    Table~\ref{tab:flip-sim} provides a Monte-Carlo simulation that empirically confirms our variance and correlation check. When the signature keeps at least $20\%$ of the tokens ($\rho\ge 0.45$), the Gaussian flip-probability differs from the exact value by at most $0.013$ in absolute terms and by $\le 11\%$ in relative terms. However, the approximation deteriorates rapidly once filtering exceeds the 80\% threshold ($\rho<0.45$).
\end{example}

Assuming the standard \textsc{Kgw} sampling model, the exact flip-probability formula (\ref{eq:exact-flip}) and its Gaussian surrogate (Proposition\,\ref{thm:flip-approx}) quantify how likely an ordinary text will be misclassified after filtering. Simulations (Table \ref{tab:flip-sim}) show that this probability remains negligible as long as i) the effective sample size passes the variance check (\ref{eq:rule-of-thumb}) and ii) the signature still retains at least 20\% of the tokens.


\begin{table}[t]
\centering
\caption{Accuracy of the Gaussian approximation for $\gamma=0.5$, $z=3$, and $n=200$ within $10^5$ Monte-Carlo runs.}
\label{tab:flip-sim}
\resizebox{\linewidth}{!}{%
\begin{tabular}{ccccccc}
\toprule
Removal & $n'\!/n$ &
$\rho$ &
Global & Approx. & Abs.\ error & Rel.\ error \\ \midrule
50\% & 0.50 & 0.71 & 0.027 & 0.028 & 0.001 & 3.7\% \\
70\% & 0.30 & 0.55 & 0.067 & 0.071 & 0.004 & 6.0\% \\
80\% & 0.20 & 0.45 & 0.118 & 0.130 & 0.012 & 10.2\% \\
85\% & 0.15 & 0.39 & 0.162 & 0.189 & 0.027 & 16.7\% \\
90\% & 0.10 & 0.32 & 0.224 & 0.276 & 0.052 & 23.2\% \\
\bottomrule
\end{tabular}}
\vspace{-0.8em}
\end{table}

}

\subsection{Proof of Theorem \ref{thm:finite-sig-bound}}
\OMIT{
Proposition\,\ref{thm:flip-approx} provides an upperbound on the flip probability. To prove Theorem\,\ref{thm:finite-sig-bound}, we will turn that bound into a \emph{safe deletion budget}, namely, the maximal number of tokens a signature can safely remove without inflating the false-positive rate beyond a chosen tolerance.}

For fixed $\gamma \in (0,1)$, $z>0$, $\varepsilon\in(0,1)$, and every finite~$n$ such that $\Phi(z)>\delta_n$, Proposition\,\ref{thm:flip-approx} asserts that
\begin{align}
    \mathbb{P}\bigl(Z'\!\ge z \mid Z<z\bigr) \;\le\; f(\rho) + \varepsilon_{\text{BE}}(n,\rho).
\end{align}
Here
$
    \varepsilon_{\mathrm{BE}}(n,\rho) \coloneqq
    (r_{n,n'}+\delta_{n})\,/\,(\Phi(z)-\delta_{n})
$
is the Berry--Esseen remainder, and
$f(\rho) \coloneqq 1-{\Phi_{2}(z,z;\rho)}/{\Phi(z)}$ with $0\le \rho\le 1$.
Since $f(\rho)$ is strictly decreasing and $f(1)=0$,
for any prescribed $\varepsilon\in(0,1)$ there exists a unique constant
$\rho_{\varepsilon}\in(0,1)$ such that $f(\rho_{\varepsilon}) \le \varepsilon/2$.
Further, any $\rho\ge\rho_{\varepsilon}$ further reduces $f(\rho)$ to at most $\varepsilon/2$.

By the definition of $\varepsilon_{\text{BE}}(n,\rho)$, we have
$\sup_{\rho\in[\rho_{\varepsilon},1]} \varepsilon_{\text{BE}}(n,\rho) = O(n^{-1/2})$.
Also, note that $\delta_n = o(1)$.
Hence, there exists an integer $N$ depending only on $\gamma, z, \varepsilon$ such that
$\varepsilon_{\text{BE}}(n,\rho)\le\varepsilon/2$ whenever
$\rho \in [\rho_{\varepsilon},1]$ and $n\ge N$.
Define the safe deletion budget as
$s_{\mathrm{safe}}(n) \coloneqq n\,(1-\rho_{\varepsilon}^{2})$.
If a filter removes $s \le s_{\mathrm{safe}}(n)$ tokens from the text,
then $\rho=\sqrt{n'/n} = \sqrt{1-s/n}\ge\rho_{\varepsilon}$,
Hence, for each $n\ge N$, it holds that
$\mathbb{P}(Z'\!\ge z \mid Z<z)
  \le
  f(\rho) + \varepsilon_{\text{BE}}(n,\rho)
  \le
  \tfrac{\varepsilon}{2} + \tfrac{\varepsilon}{2}
  =\varepsilon$.
Since $\rho_{\varepsilon}$ is constant in $n$, $s_{\mathrm{safe}}(n)$ is linear in~$n$.
Therefore $s_{\mathrm{safe}}(n) = \Theta(n)$.

\subsection{Proof of Theorem\,\ref{thm:sqrt-deterministic}}

Let $s \coloneqq \lceil c\sqrt n\rceil \le n$. 
Since the adversary knows the signature, she can compose a text of length $n$ such that precisely $s$ tokens will be filtered. Moreover, when $n$ is large enough, 
she can paint $N_g \coloneqq \gamma n + (z - y)\,\sigma_n$ of the $n-s$ non-filtered tokens green and the rest of the tokens red.
This coloring strategy yields the pre-filter z-score
$Z = (N_g-\gamma n)/\sigma_n = z-y < z.$
%
Since $n' = n-s$ and $N_g' = N_g$, the post-filter z-score is
$$
   Z'  \;=\; \frac{N_g-\gamma(n-s)}{\sigma_{n-s}}
       \;=\;
       \frac{(z-y)\sigma_n+\gamma s}{\sigma_n\sqrt{1-s/n}}
       \;=\;
       \frac{z-y + c\,{\gamma}/{\sqrt{\gamma(1-\gamma)}}}
            {\sqrt{1-{c}/{\sqrt n}}}.
$$
Note that the denominator of $Z'$ approaches 1 from below as $n\!\to\!\infty$.
Since $c > y\sqrt{(1-\gamma)/\gamma}$, the numerator of $Z'$ eventually
exceeds $z$ as $n$ increases.
Hence $Z'\ge z$ holds for all sufficiently large $n$.
As for tightness, note that deleting a red token raises the z-score by at most
$
    2\gamma/\sigma_n = \Theta(n^{-1/2}).
$
Hence, it requires $\Theta(\sqrt n)$ deletions to make a unit jump in the score,
rendering the deletion budget information-theoretically tight.

Formally, denote the pre-filter z-score of the text by
$Z_n \coloneqq {(N_g-\gamma n)}/{\sigma_n}$.
To flip the decision with the smallest possible deletion budget, the attacker must force the filter to delete only red tokens.
For each red token deleted, \(N_g\) is unchanged but the text length $n$ drops by 1.
A direct calculation gives
\begin{align}
  Z_{n-1}-Z_n
  \;=\;
  \frac{N_g-\gamma(n-1)}{\sigma_{n-1}}
  -
  \frac{N_g-\gamma n}{\sigma_n}
  \;=\;
  \frac{\gamma}{\sigma_{n-1}}
  +
  (N_g-\gamma n)\!\left(\frac{1}{\sigma_{n-1}}-\frac{1}{\sigma_n}\right).
\end{align}
Since \(N_g-\gamma n = O(n^{1/2})\) and
\(\frac{1}{\sigma_{n-1}}-\frac{1}{\sigma_n}=O(n^{-3/2})\),
the second term is \(O(n^{-1})\). For all \(n\) large enough,
\begin{align}
  Z_{n-1}-Z_n
  \;\le\;
  \frac{2\gamma}{\sqrt{\gamma(1-\gamma)}}\,
  \frac{1}{\sqrt{n}}
  \;\eqqcolon\;\frac{K}{\sqrt{n}}.
\end{align}

Suppose that the pre-filter z-score is \(Z_n = z - \delta < z\) for some fixed margin \(\delta>0\).
To reach the threshold \(z\) we must gain at least \(\delta\) units,
namely, $Z_{n-s}-Z_n \ge \delta$.
Summing the bounds on $Z_{i-1}-Z_i$ yields
\begin{align}
  Z_{n-s}-Z_n
  \;\le\;
  K\sum_{i=0}^{s-1}\frac{1}{\sqrt{n-i}}
  \;\le\;
  K\!\int_{n-s}^{n}\frac{\mathrm{d}t}{\sqrt{t}}
  \;=\;
  2K\bigl(\sqrt{n}-\sqrt{n-s}\bigr)
  \;\le\;
  \frac{2Ks}{\sqrt n},
\end{align}
where the last inequality follows from the fact
$\sqrt n-\sqrt{n-s} = {s}/({\sqrt n+\sqrt{n-s}}) \le s/\!\sqrt n$.
Thus, a necessary condition for achieving the required z-score gain \(\delta\) is
${2Ks}/{\sqrt n} \ge \delta$, which is equivalent to
$s \ge ({\delta}/{2K}) {\sqrt n}$.

In conclusion, a deterministic attack can flip the decision by deleting
\(\Omega(\sqrt{n})\) tokens.
Conversely, with \(o(\sqrt{n})\) deletions, the maximal possible gain in the
z-score remains \(o(1)\), which is insufficient to bridge the fixed gap \(\delta\).
Hence, the $\sqrt{n}$-scale deletion budget is both sufficient and necessary.

\subsection{Proof of Theorem\,\ref{thm:square-law}}

Fix an unwatermarked text $T$, and define $c_{\mathrm{safe}} \coloneqq \tfrac{z}{4}\cdot\sqrt{(1-\gamma)/{\gamma}}$. We claim that when $n=|T|$ is large enough, (a) if $s\le \lfloor c_{\mathrm{safe}}\sqrt n\rfloor$ and $z \ge z_\varepsilon \coloneqq\inf\{z>0 : \frac{1-\Phi(z)}{\Phi(2z)}\le\frac{\varepsilon}{4}\}$, then $\Perror \le \varepsilon$;
(b) if $\varepsilon \in (0, \tfrac{1}{2})$, then there exists a constant $c_{\mathrm{flip}} > c_{\mathrm{safe}}$ such that
$s \ge \lceil c_{\mathrm{flip}}\sqrt n\rceil$ implies $\Perror \ge 1-\varepsilon$.

We first prove claim (a).
Suppose that the signature has filtered
$s \le \lfloor c_{\mathrm{safe}}\sqrt n\rfloor$ tokens.
The attacker's best strategy 
to maximize $\mathbb{P}(Z'\ge z \mid Z< z)$ is to paint all the removed tokens red.
Thus, to estimate the worst-case flip probability, we can assume $N_g'=N_g$.
Recall that $N_g' \sim \mathrm{Bin}(n-s,\,\gamma)$. Let
$X \coloneqq N_g' - \gamma(n-s)$ and $\sigma_{n-s} \coloneqq \sqrt{(n-s)\gamma(1-\gamma)}$.
Note that $\mathbb E[X]=0$ and $\mathrm{Var}(X)=\sigma_{n-s}^2$. Thus, the two z-scores are 
$$
        Z  \;=\; \frac{N_g-\gamma n}{\sigma_n} \;=\; \frac{X-\gamma s}{\sigma_n},\quad
        Z' \;=\; \frac{N_g'-\gamma(n-s)}{\sigma_{n-s}}
           \;=\; \frac{N_g-\gamma(n-s)}{\sigma_{n-s}}
           \;=\; \frac{X}{\sigma_{n-s}}.
$$
Also, observe that $\gamma s \le \tfrac{1}{4}z\,\sigma_n$ holds by
the assumption $s \le \lfloor c_{\mathrm{safe}}\sqrt n\rfloor$. Thus,
$Z<z$ implies $\frac{X}{\sigma_n}
    < z + \frac{\gamma s}{\sigma_n}
    \le z + \frac{1}{4}z = \frac{5}{4}z.$
It follows that
$\mathbb{P}(Z<z)\ge\mathbb{P}(\tfrac{X}{\sigma_n}\le \tfrac{5}{4}z).$

\smallskip
Since $X$ is the sum of i.i.d.~Bernoulli random variables that are centered
and bounded by 1, the univariate Berry–Esseen bound gives
$
        \mathbb{P}\bigl(\frac{X}{\sigma_{n-s}}\le u\bigr)
        \ge
        \Phi(u)-\frac{c_{\mathrm{BE}}}{\sigma_{n-s}}
$
for any $u\in\mathbb R$. Note that
\begin{align}
  \sigma_{n-s}
  \;=\;   \sigma_n\sqrt{1-\frac{s}{n}}
  \;\ge\; \sigma_n(1-\frac{s}{2n})
  \;=\;   \sigma_n\,[1+O(\frac1{\sqrt{n}})].
\end{align}
Thus,
$\frac{X}{\sigma_n}\le \tfrac{5}{4}z$
implies $\frac{X}{\sigma_{n-s}}\le 2z$
when $n$ is large enough.
Taking $u=2z$, we have
\begin{align}
\mathbb{P}(Z<z)
        \;\ge\;
        \mathbb{P}\bigl(\frac{X}{\sigma_{n-s}}\le 2z\bigr)
        \;\ge\;
        \Phi(2z)-\frac{c_{\mathrm{BE}}}{\sigma_{n-s}}
        \;\ge\;\frac12\Phi(2z).
\end{align}
Since $Z'$ is a centered normalized binomial variable, applying the Berry–Esseen bound again yields
\begin{align}
\mathbb{P}(Z'\ge z)
   \;=\; 1-\mathbb{P}(Z'< z)
   \;\le\;
   1-\Phi(z)+\frac{c_{\mathrm{BE}}}{\sigma_{n-s}}
   \;\le\;
   1-\Phi(z)+\frac14\varepsilon\,\Phi(2z).
\end{align}
Finally, we have
\begin{align}
   \mathbb{P}(Z'\ge z\mid Z<z)
      \le
      {\mathbb{P}(Z'\ge z)}/{\mathbb{P}(Z<z)}
      \le
      {(1-\Phi(z)+\frac14\varepsilon\,\Phi(2z))}/{\frac12\Phi(2z)}
      =
      2\cdot \frac{1-\Phi(z)}{\Phi(2z)} + \frac{1}{2}\varepsilon.
\end{align}
Since $\lim_{z\to\infty}\frac{1-\Phi(z)}{\Phi(2z)}=0$
and $\frac{1-\Phi(z)}{\Phi(2z)}$ is strictly decreasing in $z$, we may pick
$
      z_\varepsilon
      \coloneqq
      \inf\{z>0 : \frac{1-\Phi(z)}{\Phi(2z)} \le \frac{\varepsilon}{4}\}
$
such that $z\ge z_\varepsilon$ implies $\Perror = \mathbb{P}(Z'\ge z\mid Z<z) \le \tfrac12\varepsilon + \tfrac12\varepsilon = \varepsilon$.
This proves part (a).

\smallskip
We now prove claim (b).
Suppose that the attacker has decided to filter $s = \lceil c\sqrt n \rceil$ tokens,
where $c$ is to be determined later.
Again, the best attacker strategy
to maximize $\mathbb{P}(Z'\ge z \mid Z< z)$ is to paint all the removed tokens red.
Therefore, we can assume that $N_g'=N_g$.
Since $\sigma_{n}=\sqrt{\gamma(1-\gamma)n}$, we have
\begin{align}
  Z'
   \;=\;\frac{N_g-\gamma n+\gamma s}{\sigma_{n-s}}
   \;=\;\frac{\sigma_n Z+\gamma s}{\sigma_{n-s}}
   \;=\;\frac{\sigma_n}{\sigma_{n-s}}
      (Z+\frac{\gamma s}{\sigma_n})
   \;=\;\frac{\sigma_n}{\sigma_{n-s}} (Z+c'\cdot(1+o(1))),
\end{align}
where $c' \coloneqq c\sqrt{\frac{\gamma}{1-\gamma}}$.
Hence, $Z'<z$ implies
$
  Z < \frac{\sigma_{n-s}}{\sigma_{n}} z - c'\cdot(1+o(1)) < z - \frac12 c'
$
for all $n$ large enough, and
\begin{align}
  \mathbb{P}\bigl(Z' < z \mid Z < z\bigr)
  &\;\le\;
  \mathbb{P}\bigl(Z<z-\tfrac12 c'\mid Z < z\bigr)
  \;=\;
  {\Phi(z-\tfrac12 c')}/{\Phi(z)}.\label{eq:cflip-bound}
\end{align}

Note that $\lim_{t\to\infty} {\Phi(z-t)}/{\Phi(z)}=0$. 
Thus, we can choose $c'$ large enough so that
${\Phi(z-c'/2)}/{\Phi(z)} < \varepsilon$.
As $c'= c\sqrt{\frac{\gamma}{1-\gamma}}$, the minimal value of $c$ to ensure this is
$
c_{\mathrm{flip}}
  \coloneqq
  \inf\{c>0:
          {\Phi(
                 z-\tfrac{c}{2}\sqrt{\tfrac{\gamma}{1-\gamma}}
               )}/{\Phi(z)}
          <\varepsilon\}.
$
Since $\Phi$ is strictly increasing, we have 
${\Phi(z-\tfrac{c}{2}\sqrt{\tfrac{\gamma}{1-\gamma}})}/{\Phi(z)}<\varepsilon$
iff
$z-\tfrac{c}{2}\sqrt{\tfrac{\gamma}{1-\gamma}} < \Phi^{-1}(\varepsilon\,\Phi(z)).$
Therefore,
$
  c_{\mathrm{flip}}
  =
  2\sqrt{\frac{1-\gamma}{\gamma}}\,
    [z-\Phi^{-1}\!(\varepsilon\,\Phi(z))].$
For any $c \ge c_{\mathrm{flip}}$, plugging $c'= c\sqrt{\frac{\gamma}{1-\gamma}}$ into \eqref{eq:cflip-bound} leads to $\mathbb{P}(Z' < z \mid Z < z) < \varepsilon$.
Consequently, 
$s \ge \lceil c_{\mathrm{flip}}\sqrt n\rceil$ implies
$\Perror = \mathbb{P}(Z'\ge z \mid Z< z) \ge 1-\varepsilon$.

\subsection{Proof of Theorem~\ref{thm:exp-universal}}
Fix an unwatermarked text $T$ and a significance level $\alpha \in (0,1)$. Let $Z\coloneqq Z_{\textrm {E}}(T)$ and $Z'\coloneqq Z_{\textrm {E}}(T')$ be the \textsc{Exp} z-scores before and after filtering (see Sec.\ \ref{sec:prelim}). Write $z\coloneqq \Phi^{-1}(1-\alpha)$. For text $T$, \textsc{Exp} is an \emph{exact} test, namely,
$\mathbb{P}(Z\ge z)=\alpha$ and $\mathbb{P}(Z< z)=1-\alpha$.
Assuming score secrecy, the retained pseudo-random scalars remain i.i.d.\ $\mathcal U[0,1]$ for $T$. Thus the retained scores are i.i.d.\ $\mathrm{Exp}(1)$ and the p-value $p_{\textrm {E}}(T')$ is $\mathcal U[0,1]$.
It follows that $\mathbb{P}(Z'\ge z)=\alpha$.
Without any independence assumption between $Z$ and $Z'$, we have
\begin{align}
\Perror
\;=\; \mathbb{P}(Z'\ge z\mid Z<z)
\;=\; {\mathbb{P}(Z'\ge z,Z<z)}/{\mathbb{P}(Z<z)}
\;\le\; {\mathbb{P}(Z'\ge z)}/({1-\alpha})
\;=\; {\alpha}/({1-\alpha}).
\end{align}


\OMIT{
Finally, we prove that $c_{\mathrm{flip}} > c_{\mathrm{safe}} = \frac{z}{4} \sqrt{\frac{1-\gamma}{\gamma}}$ when $\varepsilon < 1/2$.
Since $z>0$, we have $0<\varepsilon\,\Phi(z)<1/2$. Thus
$\Phi^{-1} \bigl(\varepsilon\,\Phi(z)\bigr) < 0$, and therefore
$z-\Phi^{-1} \bigl(\varepsilon\,\Phi(z)\bigr) > z$.
Inserting this bound into \eqref{eq:cflip-2} leads to
$c_{\mathrm{flip}}> 2z\sqrt{(1-\gamma)/\gamma}$.
Taking the ratio yields $c_{\mathrm{flip}} > 8\,c_{\mathrm{safe}} > c_{\mathrm{safe}}$.
It follows that the deletion budgets $c_{\mathrm{flip}}$ and $c_{\mathrm{safe}}$ do not overlap for reasonably small $\varepsilon$. 
}

\subsection{Safe Deletion Budget Under Unigram Coloring}

We establish Theorem \ref{thm:unigram-budget} in this subsection.
Consider an unwatermarked text with $\ell$ distinct token types, pre-filter type multiplicities $\{m_\tau\}_{\tau\in\mathcal V}$,
and post-filter type multiplicities $\{m'_\tau\}_{\tau\in\mathcal V}$. 
Then the color of each token type $\tau$ is a Bernoulli random variable $C_\tau\stackrel{\mathrm{iid}}{\sim}\mathrm{Ber}(\gamma)$ under unigram coloring.
Define
$$
Q\coloneqq \sum_\tau m_\tau^2,\quad Q'\coloneqq \sum_\tau (m'_\tau)^2,\quad
r\coloneqq \frac{Q}{n},\quad r'\coloneqq \frac{Q'}{n'}.
$$
The pre- and post-filter z-score can then be expressed by
$$
Z\;=\;\frac{\sum_\tau m_\tau(C_\tau-\gamma)}{\sqrt{\gamma(1-\gamma)\,n}},
\qquad
Z'\;=\;\frac{\sum_\tau m'_\tau(C_\tau-\gamma)}{\sqrt{\gamma(1-\gamma)\,n'}}.
$$


Note that $\mathrm{Var}(C_\tau-\gamma)=\gamma(1-\gamma)$ and that the colors are independent across token types.
Since tokens of the same type are either kept intact or moved totally after filtering, we have 
$m'_\tau\in\{0,m_\tau\}$, which implies $\sum_\tau m_\tau m'_\tau=\sum_\tau (m'_\tau)^2=Q'$.
Thus, $\mathrm{Cov}(Z,Z')=Q'/\sqrt{nn'}$.

\begin{proposition}\label{prop:uni-be-expanded}
Let a text contain $\ell$ distinct types with multiplicities $\{m_\tau\}$ and post-filter multiplicities $m'_\tau\in\{0,m_\tau\}$.
Under $\mathcal H_0$ and unigram coloring, define
$\widetilde{Z}\coloneqq {Z}/{\sqrt{r}}$ and $\widetilde{Z}'\coloneqq {Z'}/{\sqrt{r'}}.$
Then $(\widetilde{Z},\widetilde{Z}')$ is the sum of $\ell$ independent, mean‑zero $\mathbb{R}^2$ increments and has covariance
$\Sigma_\rho=\begin{psmallmatrix}1&\rho_{\mathrm{uni}}\\ \rho_{\mathrm{uni}}&1\end{psmallmatrix}$,
where $\rho_{\mathrm{uni}} = \sqrt{Q'/Q}$.
%
For any $a,b\in\mathbb{R}$, it holds that
\begin{align}\label{eq:unigram-approx}
\left|\mathbb P \big(\widetilde{Z}\le a,\widetilde{Z}'\le b\big)-\Phi_2(a,b;\rho_{\mathrm{uni}})\right|
\ \le\ C_2\,\mu_{\ell}.
\end{align}
Moreover, by defining $\beta_\ell \coloneqq \max_{\tau} {m_\tau}/\!\sqrt{Q}$, the remainder admits the explicit bound
\begin{equation}\label{eq:mu3-bound}
\mu_{\ell} \;\le\; \frac{\sqrt{2}}{\sqrt{\gamma(1-\gamma)}\,\big(1-\rho_{\mathrm{uni}}\big)^{3/2}} \Big(1+\rho_{\mathrm{uni}}^{-1}\Big)
\beta_\ell.
\end{equation}
\end{proposition}

\begin{proof}
Under unigram coloring, $\{C_\tau\}_{\tau\in\mathcal V}$ are i.i.d.\ across token types.
We can define per‑type increments
\begin{align*}
X_\tau\coloneqq \Big(w_\tau\,U_\tau,\ w'_\tau\,U_\tau\Big),\quad
U_\tau\coloneqq \frac{C_\tau-\gamma}{\sqrt{\gamma(1-\gamma)}},\quad
w_\tau\coloneqq \frac{m_\tau}{\sqrt{Q}},\quad w'_\tau\coloneqq \frac{m'_\tau}{\sqrt{Q'}}
\end{align*}
such that $(\widetilde{Z},\widetilde{Z}')=\sum_{\tau=1}^{\ell} X_\tau$ with independent, mean‑zero $X_\tau\in\mathbb{R}^2$.

Because $\mathbb{E}[U_\tau]=0$, $\mathbb{E}[U_\tau^2]=1$, we have
$\mathrm{Var}(\widetilde{Z})=\sum_\tau w_\tau^2=1$,
$\mathrm{Var}(\widetilde{Z}')=\sum_\tau (w'_\tau)^2=1$,
and
$\mathrm{Cov}(\widetilde{Z},\widetilde{Z}')=\sum_\tau w_\tau w'_\tau=\tfrac{\sum_\tau m_\tau m'_\tau}{\sqrt{QQ'}}=\sqrt{Q'/Q}=\rho_{\mathrm{uni}}$,
where we used the fact $m'_\tau\in\{0,m_\tau\}$ to obtain $\sum_\tau m_\tau m'_\tau=Q'$.
Hence, $\mathrm{Cov}\big(\sum_\tau X_\tau\big)=\Sigma_\rho$.
For any convex region $A$, the bivariate Berry--Esseen bound yields
\begin{align}
\Big|\mathbb P \big(\Sigma_\rho^{-1/2}\textstyle\sum_\tau X_\tau\in A\big)-\Phi_2(A)\Big| \;\le\; C_2\,\mu_{\ell}
\end{align}
with $\mu_{\ell}=\sum_\tau \mathbb{E}\big\|\Sigma_\rho^{-1/2}X_\tau\big\|_2^3$.
By changing variables $A\mapsto \Sigma_\rho^{-1/2}A$, we can write the above inequality as
\begin{align}
\Big|\mathbb P \big((\widetilde{Z},\widetilde{Z}')\in A\big)-\Phi_2(A;\rho_{\mathrm{uni}})\Big| \;\le\; C_2\,\mu_{\ell},
\end{align}
which yields \eqref{eq:unigram-approx} when $A=(-\infty,a]\times(-\infty,b]$.

We proceed to bound $\mu_{\ell}$.
Note that $\|\Sigma_\rho^{-1/2}\|_{\mathrm{op}}=(1-\rho_{\mathrm{uni}})^{-1/2}$ as $\Sigma_\rho$ has eigenvalues $1\pm\rho_{\mathrm{uni}}$. Hence,
\begin{align}
\big\|\Sigma_\rho^{-1/2}X_\tau\big\|_2^3
\ \le\ \frac{\|X_\tau\|_2^3}{(1-\rho_{\mathrm{uni}})^{3/2}}
\ =\ \frac{|U_\tau|^3}{(1-\rho_{\mathrm{uni}})^{3/2}}\ \big(w_\tau^2+(w'_\tau)^2\big)^{3/2}.
\end{align}
Thus, $\mathbb{E}|U_\tau|^3=\mathbb{E}|C_\tau-\gamma|^3/(\gamma(1-\gamma))^{3/2}$, and
$
\mathbb{E}|C_\tau-\gamma|^3
= \gamma(1-\gamma)\big((1-\gamma)^2+\gamma^2\big) \le \gamma(1-\gamma).
$
This leads to
\begin{align}
\mathbb{E}\big\|\Sigma_\rho^{-1/2}X_\tau\big\|_2^3
\ \le\ \frac{\big(w_\tau^2+(w'_\tau)^2\big)^{3/2}}{\sqrt{\gamma(1-\gamma)}\,(1-\rho_{\mathrm{uni}})^{3/2}}.
\end{align}
Summing over $\tau$ and exploiting the fact that $(x+y)^{3/2}\le 2^{1/2}(x^{3/2}+y^{3/2})$ for $x,y\ge 0$, we have
\begin{align}
\mu_{\ell} \ \le\ \frac{\sqrt{2}}{\sqrt{\gamma(1-\gamma)}\,(1-\rho_{\mathrm{uni}})^{3/2}}
\Big(\sum_\tau w_\tau^3+\sum_\tau (w'_\tau)^3\Big).
\end{align}
Since $\sum_\tau w_\tau^2=\sum_\tau (w'_\tau)^2=1$, it follows that
$\sum_\tau w_\tau^3\le (\max_\tau w_\tau)\sum_\tau w_\tau^2=\max_\tau w_\tau=\max_\tau m_\tau/\sqrt{Q}$, and similarly for $w'_\tau$.
Also, $m'_\tau\le m_\tau$ implies $\max_\tau w'_\tau\le \max_\tau w_\tau /\rho_{\mathrm{uni}}$. Hence
\begin{align}
\sum_\tau w_\tau^3+\sum_\tau (w'_\tau)^3 \;\le\; \max_\tau w_\tau+\max_\tau w'_\tau \;\le\; \Big(1+\rho_{\mathrm{uni}}^{-1}\Big)
\beta_\ell,
\end{align}
leading to \eqref{eq:mu3-bound}.
\end{proof}

\begin{proposition}\label{prop:uni-flip}
Under unigram coloring and the null hypothesis $\mathcal H_0$, let
\begin{align}
f(\rho,r,r';z)
\;\coloneqq \; \frac{\Phi\big(z/\sqrt{r}\big) \;-\; \Phi_2\big(z/\sqrt{r},\,z/\sqrt{r'},\,\rho\big)}{\Phi\big(z/\sqrt{r}\big)}.
\end{align}
For all $z>0$ with $\Phi(z/\sqrt{r})>\delta_\ell$, it holds that
\begin{align}
\Big|\,\mathbb{P}(Z'\!\ge z\mid Z<z)-f(\rho_{\mathrm{uni}},r,r';z)\,\Big|
\;\le\;
\varepsilon_{\mathrm{BE}}(\ell,\rho_{\mathrm{uni}}),
\end{align}
where 
$\varepsilon_{\mathrm{BE}}(\ell,\rho_{\mathrm{uni}}) \coloneqq ({C_2\,\mu_{\ell}+\delta_\ell})/({\Phi(z/\sqrt{r})-\delta_\ell})$.
\end{proposition}

\begin{proof}
Since $\mathbb P(Z\le z, Z' \le z) = \mathbb{P}(\widetilde{Z}\le z/\sqrt{r},\,\widetilde{Z}'\le z/\sqrt{r'})$,
Proposition\:\ref{prop:uni-be-expanded} yields
$\big|\mathbb{P}(Z\le z, Z' \le z)-\Phi_2(z/\sqrt{r},z/\sqrt{r'};\rho_{\mathrm{uni}})\big|
\le C_2\,\mu_{\ell}.$
For the marginal denominator $\mathbb{P}(Z<z)=\mathbb{P}(\widetilde{Z}<z/\sqrt{r})$,
the Berry–Esseen bound gives $|\mathbb{P}(\widetilde{Z} < z/\sqrt{r})-\Phi(z/\sqrt{r})|\le \delta_\ell$ with
$\delta_\ell \le  C_1\sum_\tau \mathbb{E}|w_\tau U_\tau|^3 \le  C_1\,\beta_\ell/\sqrt{\gamma(1-\gamma)}$,
noting that $\mathbb{E}\,U_\tau^2=1$ (see the proof of Proposition\:\ref{prop:uni-be-expanded}).
Combining the joint and marginal bounds exactly as in the \textsc{Kgw} case (Proposition\:\ref{thm:flip-approx}),
but with the thresholds shifted to $z/\sqrt{r}$ and $z/\sqrt{r'}$, yields the promised inequality.
\end{proof}

The quantity $\beta_\ell=\max_\tau m_\tau/\sqrt{Q}$ can be seen as the largest ``type leverage'' into the type‑mass $Q=\sum_\tau m_\tau^2$.
If no single token type carries a non‑vanishing fraction of the text as $\ell \to \infty$ (e.g., all multiplicities $m_\tau$ are within a constant factor of the median),
then $\beta_\ell=O(\ell^{-1/2})$ and thus $\delta_\ell, \mu_{\ell}=O(\ell^{-1/2})$. It follows that
$\varepsilon_{\mathrm{BE}}(\ell,\rho_{\mathrm{uni}})=O(\ell^{-1/2})$.
When there are no dominating token types, the following proposition establishes Theorem \ref{thm:unigram-budget} in the main text.
(Notice that the Berry--Esseen bound still holds as a finite-sample inequality even when there are dominating token types, but the remainder may not shrink with $\ell$.)

\begin{proposition}\label{prop:unigram-safety}
Under unigram coloring and the null hypothesis $\mathcal H_0$, assume the \textit{non-degeneracy} condition that no single token type carries a non‑vanishing fraction of the text as $\ell \to \infty$.
Then there exists a unique $\rho^\star=\rho^\star(\varepsilon;z,\gamma,r,r')\in(0,1)$ solving
$f(\rho,r,r';z)=\varepsilon/2$, such that the \emph{type-mass budget}
\begin{align}
Q'\;\ge\;\rho_\star^2\,Q \;\implies\;
\mathbb P\big(Z'\!\ge z\;\big|\;Z<z\big)\;\le\;\varepsilon
\end{align}
holds for all sufficiently large $\ell$.
Equivalently, under unigram coloring, it is safe to delete up to a $(1-\rho_\star^2)$-fraction of the pre-filter type-mass $Q$ without exceeding conditional FPR $\varepsilon$.
\end{proposition}

\begin{proof}
Given $\varepsilon\in(0,1)$, we can pick $\rho_\star\in(0,1)$ such that 
$f(\rho_\star,r,r';z)=\varepsilon/2$, and choose $\ell$ large enough that
$\varepsilon_{\mathrm{BE}}(\ell, \rho_\star) \le \varepsilon/2$ in Proposition\:\ref{prop:uni-flip}
under the non-degeneracy condition.
If $Q'\ge \rho_\star^2\,Q$, then $\rho_{\mathrm{uni}}=\sqrt{Q'/Q}\ge \rho_\star$.
Since $f(\rho,r,r';z)$ is strictly decreasing in $\rho\in(0,1)$ for fixed $(r,r',z)$, 
we have $f(\rho_{\mathrm{uni}},r,r';z)\le f(\rho_\star,r,r';z)=\varepsilon/2$. 
Adding the Berry–Esseen remainder yields $\mathbb P (Z'\!\ge z\mid Z<z)\le \varepsilon$
for all sufficiently large $\ell$.
\end{proof}


\section{\neww{Comparison with Repeated Context Masking}}
\label{app:kgw-rcm}

\begin{table}[t]
\centering
\footnotesize
\setlength{\tabcolsep}{4pt}
\caption{Focused comparison of \textsc{Kgw} and \textsc{Kgw}+RCM on Opt-1.3b. RCM context size is fixed to 8 after sweep validation. TPR and FPR are fixed-threshold rates (\%). Eligible WM is the share of watermarked samples that remain eligible after masking. Avg.\ scored WM is the mean number of scored watermarked tokens after prompt handling and any RCM masking.}
\label{tab:kgw-rcm-appendix}
\begin{tabular}{lllrrrrr}
\toprule
Dataset & $\delta$ & Method & TPR & FPR & Eligible WM & Avg.\ mask & Avg.\ scored WM \\
\midrule
C4        & 0.5 & \textsc{Kgw}      &  6.1 & 0.0 &  99.3 &  0.0 & 181.60 \\
          &     & \textsc{Kgw}+RCM  &  7.5 & 0.0 &  99.3 &  4.5 & 171.16 \\
C4        & 1.0 & \textsc{Kgw}      & 55.5 & 0.0 &  99.5 &  0.0 & 182.94 \\
          &     & \textsc{Kgw}+RCM  & 55.4 & 0.0 &  99.3 &  5.2 & 172.83 \\
C4        & 2.0 & \textsc{Kgw}      & 96.9 & 0.0 &  99.6 &  0.0 & 185.23 \\
          &     & \textsc{Kgw}+RCM  & 95.7 & 0.0 &  99.7 &  6.7 & 171.88 \\
HumanEval & 0.5 & \textsc{Kgw}      &  2.4 & 0.0 & 100.0 &  0.0 & 199.12 \\
          &     & \textsc{Kgw}+RCM  &  0.6 & 0.0 &  95.7 & 51.9 &  94.96 \\
HumanEval & 1.0 & \textsc{Kgw}      &  9.8 & 0.0 & 100.0 &  0.0 & 194.29 \\
          &     & \textsc{Kgw}+RCM  &  6.1 & 0.0 &  96.3 & 51.8 &  92.57 \\
HumanEval & 2.0 & \textsc{Kgw}      & 43.3 & 0.0 & 100.0 &  0.0 & 196.09 \\
         &     & \textsc{Kgw}+RCM  & 32.9 & 0.0 &  94.5 & 47.7 &  99.67 \\
MBPP      & 0.5 & \textsc{Kgw}      & 12.8 & 1.6 &  93.5 &  0.0 & 142.90 \\
          &     & \textsc{Kgw}+RCM  &  6.3 & 0.0 &  93.0 & 21.9 & 110.21 \\
MBPP      & 1.0 & \textsc{Kgw}      & 52.4 & 1.6 &  95.2 &  0.0 & 149.47 \\
          &     & \textsc{Kgw}+RCM  & 40.2 & 0.0 &  92.1 & 20.7 & 104.98 \\
MBPP      & 2.0 & \textsc{Kgw}      & 83.5 & 1.6 &  93.3 &  0.0 & 148.21 \\
          &     & \textsc{Kgw}+RCM  & 79.6 & 0.0 &  92.3 & 17.7 & 108.78 \\
\bottomrule
\end{tabular}
\end{table}

\neww{Repeated context masking (RCM) is a non-distortion mechanism to avoid repeated watermark bias, implemented in watermarks like SynthID-Text \cite{dathathri2024scalable}. In this section, we treat \textsc{Kgw}+RCM as a variant of \textsc{Kgw} and compare it with plain \textsc{Kgw}. Note that RCM is not a pure detector-side method: during generation, it suppresses the \textsc{Kgw} bias whenever the recent left context has already appeared earlier in the same text, and during detection, it excludes those repeated context positions before the z-test.}

\neww{We compare \textsc{Kgw}+RCM and \textsc{Kgw} on Opt-1.3b over C4, \textsc{HumanEval}, and \textsc{Mbpp}. We select the RCM context size on a validation sweep and use context size~8, which gives the best mean TPR without increasing mean FPR. Table~\ref{tab:kgw-rcm-appendix} reports the main results.
It turns out that RCM yields only a modest gain on C4 at the weakest setting (i.e., \(\delta = 0.5\), where TPR increases from $6.1\%$ to $7.5\%$) and is otherwise neutral or slightly worse on C4. On the low-variation code benchmarks, RCM reduces TPR across all tested watermark strengths. This TPR reduction is likely due to substantial masking of scoreable positions, which causes the resulting loss of green evidence to outweigh the benefit from suppressing repeated context bias under the fixed \textsc{Kgw} z-threshold. More in-depth evaluation and analysis remain an intriguing future direction.}

\bibliographystyle{cas-model2-names}
\bibliography{ipm}

\end{document}